\documentclass{article}
\usepackage{PRIMEarxiv}
\usepackage{natbib}

\usepackage[utf8]{inputenc} 
\usepackage[T1]{fontenc}    
\usepackage[hyperfootnotes=true]{hyperref}       
\usepackage{url}            
\usepackage{booktabs}       
\usepackage{amsfonts}       
\usepackage{nicefrac}       
\usepackage{microtype}      
\usepackage{xcolor}         

\title{Adaptive Prior Selection in Gaussian Process Bandits with Thompson Sampling}

\author{%
    Jack Sandberg,
  Morteza Haghir Chehreghani \\
  Department of Computer Science and Engineering\\
  Chalmers University of Technology and University of Gothenburg\\
  Gothenburg, Sweden \\
  \texttt{\{jack.sandberg, morteza.chehreghani\}@chalmers.se}
}

\hypersetup{
    colorlinks, 
    linkcolor={red!50!black}, 
    citecolor={blue!50!black}, 
    urlcolor={blue!80!black} 
}
\usepackage{adjustbox}
\usepackage{multirow}
\usepackage[textsize=tiny]{todonotes}
\usepackage{layouts}
\usepackage{wrapfig}
\usepackage{mathtools}
\usepackage{amsthm}
\usepackage{amsmath}
\usepackage{amssymb}
\usepackage[capitalize]{cleveref}
\usepackage{crossreftools} 
\usepackage{algorithm}
\usepackage[noend]{algorithmic}
\usepackage{tikz}
\usetikzlibrary{positioning}
\usetikzlibrary{calc} 
\usetikzlibrary{arrows}

\usepackage{dsfont}
\usepackage{comment}
\usepackage{bm}
\usepackage{subcaption}
\usepackage{pgfmath}
\usepackage{pifont}

\usepackage{thm-restate, thmtools}
\usepackage{mathtools}

\DeclarePairedDelimiterX{\norm}[1]{\lVert}{\rVert}{#1} 

\usepackage[normalem]{ulem}
\usepackage{threeparttable}

\renewcommand{\paragraph}[1]{\noindent\textbf{#1}}

\renewcommand{\algorithmicensure}{\textbf{return}}
\renewcommand{\algorithmiccomment}[1]{\hfill $\triangleright$ #1}

\newcommand{\alglinelabel}{%
  \addtocounter{ALC@line}{-1}
  \refstepcounter{ALC@line}
  \label
}
\crefname{ALC@line}{line}{lines}
\AddToHook{cmd/appendix/before}{%
    \crefalias{section}{appendix}%
    \crefalias{subsection}{appendix}
}

\DeclareMathOperator*{\argmax}{arg\,max} 

\declaretheorem[sibling=theorem]{definition,corollary,proposition,lemma,conjecture,assumption,remark}

\providecommand{\customname}{Theorem}

\crefname{assumption}{Assumption}{Assumptions}

\newcommand{\C}{{\mathcal{C}}}

\newcommand{\E}{{\mathbb{E}}}
\newcommand{\Eb}{{\mathbb{E}}}

\newcommand{\f}{{\mathbf{f}}}

\newcommand{\GP}{{\mathcal{GP}}}

\renewcommand{\k}{{\mathbf{k}}}
\newcommand{\K}{{\mathbf{K}}}

\newcommand{\N}{{\mathcal{N}}}
\renewcommand{\O}{{\mathcal{O}}}
\newcommand{\Pb}{{\mathbb{P}}}

\newcommand{\R}{{\mathbb{R}}}

\newcommand{\V}{{\mathbb{V}}}

\newcommand{\Xc}{{\mathcal{X}}}

\newcommand{\y}{{\mathbf{y}}}

\newcommand{\beps}{{\bm{\epsilon}}}

\newcommand{\bmu}{{\bm{\mu}}}

\newcommand{\mumax}{\mu_{\text{max}}}

\newcommand{\ind}{{\perp\!\!\!\!\perp}}

\newcommand{\BR}{{\text{BR}}}

\newcommand{\1}{\mathds{1}}
\newcommand{\indfcn}[1]{\mathds{1}\{#1\}}

\newcommand{\sumt}{{\sum_{t \in [T]}}}
\newcommand{\sump}{{\sum_{p \in P}}}
\newcommand{\maxt}{\max_{t \in [T]}}

\let\emptyset\varnothing

\newcommand{\eqd}{\overset{d}{=}}

\begin{document}

\maketitle

\begin{abstract}
    Gaussian process (GP) bandits provide a powerful framework for performing blackbox optimization of unknown functions. The characteristics of the unknown function depend heavily on the assumed GP prior. Most work in the literature assume that this prior is known but in practice this seldom holds. Instead, practitioners often rely on maximum likelihood estimation to select the hyperparameters of the prior - which lacks theoretical guarantees. In this work, we study two algorithms for joint prior selection and regret minimization in GP bandits based on GP Thompson sampling (GP-TS): Prior-Elimination GP-TS (PE-GP-TS) that disqualifies priors with poor predictive performance, and HyperPrior GP-TS (HP-GP-TS) that utilizes a bi-level Thompson sampling scheme. We theoretically analyze the algorithms and establish a sublinear regret bound for HP-GP-TS. In addition, we demonstrate the effectiveness of these algorithms compared to the alternatives through extensive experiments with synthetic and real-world data.
\end{abstract}
\section{Introduction} \label{sec:introduction}
The Gaussian process bandit problem is 
a variant of the multi-armed bandit problem where the arms are correlated and their expected reward is sampled from a Gaussian process (GP). The flexibility of GPs have made GP bandits applicable in a wide range of areas that need to optimize blackbox functions with noisy estimates, 
including hyperparameter tuning \citep{turnerBayesian2021a}, 
online advertising \citep{nuaraCombinatorialBanditAlgorithmOnline2018}, and portfolio optimization \citep{gonzalvezFinancial2019}.
Most of the theoretical results in the literature assume that the GP prior is known but this is seldom the case in practical applications. Even with expert domain knowledge, selecting the exact prior to use can be a difficult task. Most practitioners tend to utilize maximum likelihood estimation (MLE) to identify suitable prior parameters. However, in a sequential decision making problem MLE is not guaranteed to recover the correct parameters which can hurt the performance.

As summarized in \cref{tab:comparison}, previous works by \citet{wangTheoretical2014,berkenkampNoRegret2019} propose algorithms that use a decreasing sequence of lengthscales according to a fixed schedule. A drawback of these schedules is that they cannot adapt to the data and may therefore explore excessively. The Lengthscale Balancing GP-UCB algorithm of \citet{ziomekBayesian2024} selects lengthscales such that each selected lengthscale incurs a similar amount of regret. However, this scheme relies on knowing the regret bounds, which can be impractical. \citet{ziomekTimevarying2025,luSurrogate2023} propose algorithms that support unknown priors of (finite) arbitrary type. Prior-Elimination GP-UCB (PE-GP-UCB) \citep{ziomekTimevarying2025} selects the prior and arm that maximize a joint upper confidence bound and eliminates priors with poor predictive performance. 
The joint upper confidence bound induces a double optimism in PE-GP-UCB that can lead to extra exploration.  
EGP-TS \citep{luSurrogate2023} uses bi-level Thompson sampling to select both a prior and an arm according to their posterior probabilities of being the true prior and optimal arm, respectively. Among these methods, posterior sampling is the only data-adaptive prior selection rule, and provides the closest analog to MLE.

EGP-TS is an instantiation of the more general MixTS algorithm \citep{hongThompson2022}, whose regret was analyzed in the standard bandit and linear setting. However, the theoretical analyses for both algorithms are flawed. The technical issues in the regret analysis of EGP-TS were recently demonstrated by \citet{sandbergComments2026} and, as we show in this work, the analysis of MixTS in the linear setting contains separate technical issues that invalidate the regret bound of \citet{hongThompson2022}.

\newcommand{\tablerow}[5]{#2 & #1 & #4 & #3 & #5 \\}
\begin{table}
    \centering
    \caption{Comparison of similar work in GP bandits with an unknown prior.}%
    \label{tab:comparison}%
\resizebox{\textwidth}{!}{
\begin{threeparttable}
    \begin{tabular}{l l l l l l c c} \toprule
         \tablerow{Algorithm}{Work}{MIG dependence}{Prior selection}{Supports unknown} \midrule
        \tablerow{BOHO (EI)}{\citep{wangTheoretical2014}}{$\hat{\gamma}_{T}^{3/2}$}{Schedule}{Lengthscale}
        \tablerow{A-GP-UCB}{\citep{berkenkampNoRegret2019}}{${\gamma}_{T, p_T}$\tnote{\dag}}{Schedule}{Lengthscale and RKHS norm}
        \tablerow{EGP-TS}{\citep{luSurrogate2023}}{$\sqrt{|P|\hat{\gamma}_{T}}$ (invalid)}{Posterior sampling}{Arbitrary mean and kernel} 
        \tablerow{LB-GP-UCB}{\citep{ziomekBayesian2024}}{$\gamma_{T,\bar{p}}$\tnote{\dag}}{Regret balancing}{Lengthscale and RKHS norm}
        \tablerow{PE-GP-UCB}{\citep{ziomekTimevarying2025}}{$\sqrt{|P| \hat{\gamma}_{T}}$}{Optimistic}{Arbitrary mean and kernel}
        \tablerow{{\bf PE-GP-TS}}{{\bf This work}}{$\sqrt{|P| \hat{\gamma}_T}$}{Optimistic}{Arbitrary mean and kernel}
        \tablerow{{\bf HP-GP-TS}\tnote{\ddag}}{{\bf This work}}{$\sqrt{\bar{\gamma}_T(P_1)}$}{Posterior sampling}{Arbitrary mean and kernel}
        \bottomrule
    \end{tabular}%
\begin{tablenotes}
    \item[\dag] $p_T$ is the final prior selected by A-GP-UCB, and $\bar{p}$ is the prior that minimizes the frequentist regret of GP-UCB. 
    \item[\ddag] Equivalent to EGP-TS \citep{luSurrogate2023}, we refer to it as HP-GP-TS.
\end{tablenotes}
\end{threeparttable}
}
\end{table}
Motivated by the excessive exploration of double optimism, alongside the flawed theoretical guarantees of existing Thompson sampling approaches, we investigate two distinct TS-based algorithms for GP-bandits with unknown priors. 
The first algorithm, Prior-Elimination GP-TS (PE-GP-TS), is an extension of PE-GP-UCB that replaces the doubly optimistic selection rule with posterior sampling and one less layer of optimism. We analyze the regret of PE-GP-TS and obtain a regret bound of order $\O(\sqrt{T |P|\hat{\gamma}_T \log T})$ (which matches that of PE-GP-UCB) plus a term (left unbounded) depending on the uncertainty of the optimal arm under the correct prior. Here, $T$ is the horizon, $|P|$ is the number of priors and $\hat{\gamma}_T$ is the worst-case maximum information gain.
The second algorithm we study is EGP-TS, which we refer to as HyperPrior GP-TS (HP-GP-TS) to emphasize its use of a hyperprior, and it removes both levels of optimism. 
Our analysis of HP-GP-TS addresses the issues in the previous work and yields a regret bound of order $\O(\sqrt{T \bar{\gamma}_T(P_1) \log T})$ where $\bar{\gamma}_T(P_1)$ is a sum of maximum information gains with cardinality equal to the horizon $T$ times the hyperprior probability $P_1(\cdot)$ s.t. $\bar{\gamma}_T(P_1) < |P| \hat{\gamma}_T$ generally holds.

We evaluate our methods on three sets of synthetic experiments and three experiments with real-world data. Across the experiments, the Thompson sampling based methods outperform PE-GP-UCB. Additionally, we find that the regret of HP-GP-TS does not increase with $|P|$ in our scaling experiments. Finally, we analyze the priors selected by the algorithms and observe that HP-GP-TS selects the correct prior more often than the other algorithms.  

The contributions of this work can be summarized as:

\begin{itemize}
    \item We propose a Thompson sampling based algorithm for GP-bandits with unknown prior, PE-GP-TS, and theoretically analyze its regret.
    
    \item We provide a sublinear regret bound for HP-GP-TS \citep[EGP-TS]{luSurrogate2023} that depends on $\bar{\gamma}_T(P_1)$, correcting and improving upon the bound of \citet{luSurrogate2023}. 
    \item We identify technical issues with the proof of the regret bound for MixTS \citep{hongThompson2022} in the linear setting, preventing its direct extension to the GP-setting.
    \item We experimentally evaluate the TS-based algorithms on both synthetic and real-world data, demonstrating that they achieve competitive performance and that the regret of HP-GP-TS does not empirically increase with $|P|$.
\end{itemize}

\section{Background and problem statement} \label{sec:background}
\paragraph{Problem statement}
We consider a sequential decision making problem where an agent repeatedly selects among a set of arms and receives a random reward whose mean depends on the selected arm and is unknown to the agent. The goal of the agent is to maximize the cumulative sum of rewards over a finite time horizon. We assume that the distribution of the means, the {\it prior}, is sampled from a set of priors, the {\it hyperprior}. An effective agent must distinguish which prior the means are sampled from to ensure it explores efficiently. 

Now, let us formally state the problem. Let $\Xc \subseteq [0,r]^d$ denote the finite set of arms and $P$ a finite set of priors with associated prior mean and kernel functions $\mu_{1, p}: \Xc \mapsto \R$ and $k_{1, p}: \Xc \times \Xc \mapsto [-1, 1]$, $\forall p \in P$. Let $p^* \in P$ denote the true prior and assume the expected reward function $f: \Xc \mapsto \R \sim \GP(\mu_{1,p^*}, k_{1,p^*})$ is a sample from a Gaussian process with prior $p^*$. Both the function $f$ and the true prior $p^*$ are considered unknown. We will consider two settings: In the frequentist selection setting, the prior $p^* \in P$ is picked arbitrarily. In the Bayesian selection setting, the prior is sampled from a known hyperprior $p^* \sim P_1$.
 Then, for time step $t = 1, 2, \ldots, T$ where $T$ is the horizon, the agent selects an arm $x_t \in \Xc$ and observes the reward $y_t = f(x_t) + \epsilon_t$ where $\{\epsilon_t\}_{t=1}^T$ are i.i.d. zero-mean Gaussian noise with variance $\sigma^2$. The goal of the agent is to select a sequence of arms $\{ x_t \}_{t=1}^T$ that minimizes the regret $R(T) = \sum_{t \in [T]} f(x^*) - f(x_t)$ where $[T] = \{1,\ldots, T\}$ and $x^* = \argmax_{x \in \Xc} f(x)$. In the Bayesian selection setting, we evaluate the agent based on the Bayesian regret $\BR(T) = \E \left[ R(T) \right]$ where the expectation is taken over the prior $p^*$, the expected reward function $f$, the noise $\{\epsilon_t\}_{t=1}^T$ and the (potentially) stochastic selection of arms.

\paragraph{Gaussian processes} A Gaussian process $f(x) \sim \GP(\mu, k)$ is a collection of random variables such that for any subset $\{ x_1, \ldots, x_n\} \subset \Xc$, the vector $[f(x_1), \ldots, f(x_n)] \in \R^n$ has a multivariate Gaussian distribution. The probabilistic nature of GPs make them very useful for defining and solving bandit problems where the arms are correlated.
Given the history $H_t = \{(x_i, y_i)\}_{i=1}^{t-1}$,
the posterior mean and kernel functions of a Gaussian process $\GP(\mu, k)$ are given by 
    $\mu_t(x) = \mu(x) + \k^\top \left( \K + \sigma^2 I \right)^{-1} (\y - \bmu)$, and $
    k_t(x, \tilde{x}) = k(x, \tilde x) - \k^\top \left( \K + \sigma^2 I \right)^{-1} \tilde{\k}$.
Above, $\k, \tilde \k \in \R^{t-1}$ are vectors such that $(\k)_i = k(x_i, x)$ and $(\tilde{\k})_i = k(x_i, \tilde x)$. Additionally, $\y, \bmu \in \R^{t-1}$ are also vectors such that $(\y)_i = y_i$ and $(\bmu)_i = \mu(x_i)$. The gram matrix is denoted by $\K \in \R^{t-1 \times t-1}$ where $(\K)_{i,j} = k(x_i, x_j)$. Let $\mu_{t, p}$ and $k_{t, p}$ denote the posterior mean and kernel for a Gaussian process with prior $p \in P$ at time $t$ and let $\sigma^2_{t,p}(x) = k_{t, p}(x, x)$ denote the posterior variance at time $t$. The kernel $k$ determines important characteristics of the functions $f$, see \cref{app:kernels} for more details and examples.

\begin{figure}
    \centering
    \begin{subfigure}[b]{0.35\linewidth}
        \includegraphics[width=\linewidth]{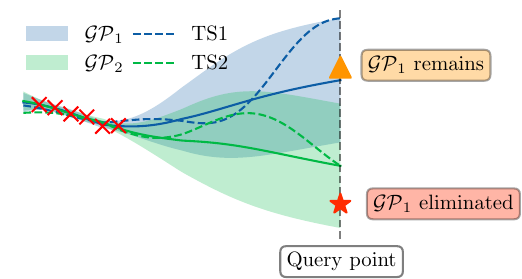}
        \caption{}
        \label{fig:priorelims}
    \end{subfigure}%
    \begin{subfigure}[b]{0.5\linewidth}
        \includegraphics[width=\linewidth]{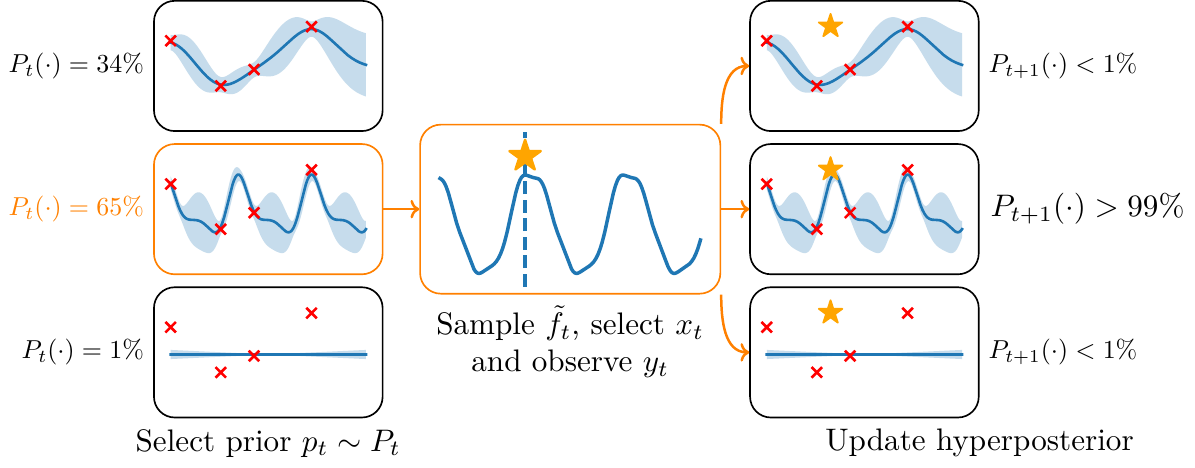}
        \caption{}
        \label{fig:hpgpts}
    \end{subfigure}
    \caption{a) Elimination procedure of PE-GP-TS. The solid lines correspond to posterior means and the shaded regions are confidence intervals. The figure has been adapted from \citet{ziomekTimevarying2025}. The dashed lines are samples from the posteriors. 
    b) Overview of HP-GP-TS. 
    }
\end{figure}

\paragraph{Information gain} The maximal information gain (MIG) is a measure of the reduction in uncertainty of $f$ after observing the most informative data points up to a specified size. The MIG commonly occurs in regret bounds for GP bandit algorithms \citep{srinivasInformationTheoretic2012,vakiliInformationGainRegret2021} and its growth rate is strongly determined by the prior kernel of the GP. Hence, we will define the MIG for any fixed GP prior $p \in P$. Let $\y_A$ denote noisy observations of $f$ at the locations $A \subset \Xc$. Then, the MIG given prior $p \in P$ is defined as 
    $\gamma_{T, p} := \sup_{A \subset \Xc, |A| \leq T} I_p(\y_A; f),$
where $I_p(\y_A; f) = H(\y_A | p) - H(\y_A | f, p)$ is the mutual information between $\y_A$ and $f$ given $p$, and $H(\cdot)$ denotes the entropy. To aid our analysis later, we also define the worst-case MIG as $\hat{\gamma}_T := \max_{p \in P} \gamma_{T,p}$ and the hyperprior-weighted MIG as $\bar{\gamma}_T(P_1) := \sump \Gamma_p(T P_1(p))$ for concave $\Gamma_p(\cdot)$ s.t. $\Gamma_p(t) \geq \gamma_{t,p}$ for all $t,p \in [T] \times P$. For the RBF and Matérn kernels, $\gamma_{T,p} = \O ( \log^{d+1}(T) )$ and $\gamma_{T,p} = \O ( T^{\frac{d}{2\nu + d}} \log^{\frac{2\nu}{2\nu + d}}(T) )$ \citep{srinivasInformationTheoretic2012,vakiliInformationGainRegret2021}.

\section{Algorithms} \label{sec:algorithms}
As discussed by \citet{russoLearningOptimizePosterior2014}, TS can offer advantages over UCB algorithms for problems where constructing tight confidence bounds is difficult. In addition, Thompson sampling is often observed to perform better than UCB in practice \citep{chapelleEmpiricalEvaluationThompson2011,wenEfficientLearningLargeScale2015,kandasamyParallelisedBayesianOptimisation2018,akerblomOnlineLearningNetwork2023,akerblomOnlineLearningEnergy2023}. Motivated by this, we present two algorithms for adaptive prior selection based on TS.

\begin{wrapfigure}[21]{r}{0.5\linewidth}
\begin{minipage}[t]{\linewidth}
\vspace*{-2.5em}
\begin{algorithm}[H] 
    \caption{Prior Elimination GP-TS (PE-GP-TS)} \label{alg:PEGPTS}
    \begin{algorithmic}[1]
        \INPUT Horizon $T$, prior functions $\{\mu_{1, p}, k_{1, p}\}_{p \in P}$, confidence parameters $\{\beta_t\}_{t=1}^T$ and $\{\xi_t\}_{t=1}^T$.
        \STATE $P_1 = P$, $S_{0, p} = \emptyset$ $\forall p \in P$
            \FOR{$t = 1, 2\ldots, T$}
                \STATE Sample $\tilde{f}_{t, p} \sim \GP(\mu_{t,p}, k_{t,p})$ $\forall p \in P_{t}$
                \STATE Set $x_t, p_t = \argmax_{x, p \in \Xc \times P_{t}} \tilde{f}_{t, p}(x)$
                \STATE $S_{t, p_t} = S_{t-1,p_t} \cup \{ t \}$ and $S_{t, p} = S_{t-1, p}$ for $p \in P \setminus \{p_t\}$
                \STATE Observe $y_t = f(x_t) + \epsilon_t$
                \STATE Set $\eta_t = y_t - \mu_{t, p_t}(x_t)$ \alglinelabel{line:predictiveerror}
                \STATE Set $V_t = \sqrt{\xi_t | S_{t, p_t} |} + \sum_{i \in S_{t, p_t}} \sqrt{\beta_i} \sigma_{i, p_t}(x_i)$
                \IF{$ \left| \sum_{i \in S_{t,p_t}} \eta_i \right| > V_t$ and $|P_t| > 1$} \alglinelabel{line:elimincriteria}
                    \STATE $P_{t+1} = P_{t} \setminus \{p_t\}$
                \ELSE
                    \STATE $P_{t+1} = P_{t}$
                \ENDIF
            \ENDFOR
    \end{algorithmic}
\end{algorithm}
\end{minipage} 
\end{wrapfigure}

\subsection{Prior-Elimination with Thompson sampling} \label{sec:pegpts}

Our first algorithm is an extension of PE-GP-UCB \citep{ziomekTimevarying2025} to be employed with Thompson sampling -- instead of UCB. 
The key difference is that instead of maximizing the upper confidence bound $U_{t}(x, p) = \mu_{t,p}(x) + \sqrt{\beta_t} \sigma_{t,p}(x)$ over $\Xc \times P_{t}$, we instead sample $\tilde{f}_{t, p}$ from the posterior $\GP(\mu_{t,p}, k_{t,p})$ for all priors $p \in P_{t}$ where $P_{t}$ is the set of active priors. Then, we select the arm and prior $x_t, p_t$ such that $x_t, p_t = \argmax_{x, p \in \Xc \times P_{t}} \tilde{f}_{t, p}(x)$. Whilst PE-GP-UCB has two layers of optimism, the upper confidence bound and joint maximization of $x$ and $p$, PE-GP-TS has only a single layer of optimism - which should alleviate potential overexploration issues.

The elimination procedure of PE-GP-TS is illustrated in \cref{fig:priorelims}. Samples $\tilde{f}_{t, p}$ are drawn from the active prior $p \in P_{t}$. Then, the unknown function $f$ is queried at the selected arm $x_t$. If the observed value differs too much from the prediction made by the selected prior, then the selected prior is eliminated. Otherwise, it remains active.

The PE-GP-TS algorithm is presented in \cref{alg:PEGPTS}. Similar to PE-GP-UCB, the set $S_{t, p}$ is used to store the time steps where prior $p$ was selected up to and including time $t$. When prior $p_t$ is selected, the prediction error $\eta_t = y_t - \mu_{t, p_t}(x_t)$ between the observed and predicted value made by the prior $p_t$ is computed. If the sum of prediction errors made by the prior $p_t$ exceeds the threshold value $V_t$, then $p_t$ is eliminated from the active priors $P_t$, see \cref{line:elimincriteria}. Note that at time step $t$, only the selected prior $p_t$ can be eliminated. As such, if a prior is very pessimistic it may never be selected and therefore will never be eliminated. Thus, the final set of active priors $P_T$ should be viewed as non-eliminated priors rather than necessarily being reasonable priors.

{\parfillskip0pt\par}
\begin{wrapfigure}[10]{r}{0.5\linewidth}
\begin{minipage}{\linewidth}
\vspace*{-4.5em}
\begin{algorithm}[H]
    \caption{HyperPrior GP-TS (HP-GP-TS)} \label{alg:HPGPTS}
    \begin{algorithmic}[1]
        \INPUT Horizon $T$, prior functions $\{\mu_{1, p}, k_{1, p}\}_{p \in P}$, hyperprior $P_1$.
            \FOR{$t = 1, 2\ldots, T$}
                \STATE Sample $p_t \sim P_{t}$
                \STATE Sample $\tilde{f}_{t} \sim \GP(\mu_{t, p_t}, k_{t, p_t})$
                \STATE Set $x_t = \argmax_{x \in \Xc} \tilde{f}_t(x)$
                \STATE Observe $y_t = f(x_t) + \epsilon_t$
                \STATE Set ${P}_{t+1}(p) \propto \mathbb{P}(y_t | x_t, \{ x_i, y_i \}_{i = 1}^{t-1}, p) \cdot P_{t}(p)$ 
                 \COMMENT{Update hyperposterior}
            \ENDFOR
    \end{algorithmic}
\end{algorithm}
\end{minipage}
\end{wrapfigure}

\subsection{HyperPrior Thompson sampling} \label{sec:hpgpts}
In our first algorithm, we removed one layer of optimism. The second algorithm we study is a fully Bayesian algorithm that uses a hyperposterior sampling scheme where both the prior and the mean function are sampled from their respective posteriors. By shedding the optimism over the selected prior $p_t$, HP-GP-TS should be able to avoid costly exploration by selecting likely priors instead of optimistic ones.

The algorithm is visualized in \cref{fig:hpgpts} and 
presented in detail in \cref{alg:HPGPTS}. In the first step, the current prior $p_t$ is sampled from the hyperposterior $P_{t}$. Then, a single sample $\tilde f_t$ is taken from the selected posterior $\GP(\mu_{t, p_t}, k_{t, p_t})$ and is used to select the current arm: $x_t = \argmax_{x \in \Xc} \tilde{f}_t(x)$. After observing $y_t$, the hyperposterior is updated by computing the likelihood of $y_t$ under the different priors. Note that since the set of priors $P$ is finite, computing the posterior is tractable albeit computationally costly for large $t$ with a complexity of $\O(t^3|P|)$. The algorithm can be extended to continuous priors $P$ using MCMC sampling. In comparison to SCoreBO \citep{hvarfner_self-correcting_2023} and other fully Bayesian algorithms that compute expected values over the hyperposterior through sampling, HP-GP-TS requires only one sample from the posterior and hyperposterior -- potentially reducing the computational cost significantly. The likelihood $\mathbb{P}(y_t | x_t, \{ x_i, y_i \}_{i = 1}^{t-1}, p) = \N(y_t; \mu_{t, p}(x_t), \sigma_{t,p}^2(x_t) + \sigma^2)$ is simply the Gaussian likelihood of the posterior at $x_t$ with added Gaussian noise with variance $\sigma^2$.

\section{Regret analysis} \label{sec:regretanalysis}
In this section, we analyze the regret for the proposed algorithms. Recall from the problem statement that we consider two slightly different settings for the two algorithms. Specifically, for PE-GP-TS we assume the unknown prior $p^*$ is selected arbitrarily from $P$ whilst for HP-GP-TS we assume that the unknown prior $p^*$ is selected from a known hyperprior distribution $P_1$.

\subsection{Analysis of PE-GP-TS}
\citet{ziomekTimevarying2025} structured the proof of the regret bound for PE-GP-UCB into 4 larger steps; First, showing that $p^*$ is never eliminated with high probability. Second, establishing a bound on the instantaneous regret. Third, bounding the cumulative regret. Finally, the cumulative bound is re-expressed in terms of the worst-case MIG. For PE-GP-TS, we establish a new bound on the instantaneous regret and then adapt the steps of \citeauthor{ziomekTimevarying2025} to accommodate the new bound.
To bound the instantaneous regret in the lemma below, we require concentration inequalities to hold for the posteriors, the posterior samples and the noise (see \cref{lem:concentrationintervals,lem:noiseconc}).
\begin{restatable}{lemma}{lemSimpleRegret}
If the events of \cref{lem:concentrationintervals,lem:noiseconc} holds, then the following holds for the instantaneous regret of PE-GP-TS for all $t \in [T]$:
$
    f(x^*) - f(x_t) \leq 2 \sqrt{\beta_t} \sigma_{t, p^*}(x^*) + \sqrt{\beta_t} \sigma_{t, p_t}(x_t) - \eta_t + \epsilon_t. \label{eq:simpleregret}
$
\end{restatable}

Compared to the instantaneous regret bound for PE-GP-UCB, we obtain the additional term $2 \sqrt{\beta_t} \sigma_{t, p^*}(x^*)$ which leads to the following regret bound:

\begin{restatable}{theorem}{thmPEGPTSregretbound} \label{thm:pegpts}
    Let $B_{p^*} = \beta_1 + \sup_{x \in \Xc} |\mu_{1, p^*}(x)|$ and $C = 2 /\log(1 + \sigma^{-2})$. If $p^* \in P$ and $f \sim \GP(\mu_{1,p^*}, k_{1,p^*})$, then PE-GP-TS with confidence parameters $\beta_t = 2 \log (2 |\Xc| |P| \pi^2 t^2 / 3 \delta)$ and $\xi_t = 2 \sigma^2 \log (|P| \pi^2 t^2 / 3\delta)$, satisfies the following regret bound with probability at least $1 - \delta$:
    \begin{equation}
        R(T) \leq 2 |P| B_{p^*} + 2 \sqrt{\xi_T |P| T} + 2 \sqrt{C T \beta_T \hat{\gamma}_{T} |P|}
        + 2 \sqrt{T \beta_T \textstyle \sumt \sigma^2_{t,p^*}(x^*)}
    \end{equation}
\end{restatable}

The bound of the first three terms is of order $\O( \sqrt{T \beta_T \hat{\gamma}_T})$ w.r.t. $T$ which matches that of PE-GP-UCB. To our knowledge, the best lower bound for standard GP bandits in the Bayesian setting, where $f$ is sampled from a GP, is $\Omega(\sqrt{T})$ for $d = 1$ \citep{scarlettTight2018}. This would suggest that our bound is tight up to a factor $\O(\sqrt{\beta_T \hat{\gamma}_T})$ when considering only the first three terms. However, note that the sublinearity of $\sumt \sigma_{t,p^*}^2(x^*)$ is not demonstrated.

\subsection{Analysis of HP-GP-TS}
We analyze the regret of HP-GP-TS by decomposing it into three terms and using the prior confidence technique.
The initial regret decomposition is similar to \citet{luSurrogate2023} and the prior confidence technique is first employed by \citet{hongThompson2022} in standard and linear settings. However, as we discuss in \cref{sec:comparison}, both of these works have fundamental issues making their theoretical analyses invalid.

First, note that HP-GP-TS inherits the probability matching property of GP-TS that $x_t | H_t \eqd x^* | H_t$ where $\eqd$ denotes equal in distribution. In addition, $p_t | H_t \eqd p^* | H_t$ since $p_t$ is sampled from the posterior distribution of $p^*$. Using this, one can derive the following decomposition of the regret:
\begin{equation}
    \BR(T) = \sumt \E[
    \underbrace{f(x^*) - U_{t,p^*}(x^*)}_{(1)}
    +
    \underbrace{(\sqrt{\beta_t} + \sqrt{\eta_T}) \sigma_{t,p_t}(x_t)}_{(2)}
    +
    \underbrace{L_{t,p_t}(x_t) - f(x_t)}_{(3)}
]
\end{equation}
where the upper confidence bound $U_{t,p}(x) = \mu_{t,p}(x) + \sqrt{\beta_t} \sigma_{t,p}(x)$ and the lower confidence bound $L_{t,p}(x) = \mu_{t,p}(x) - \sqrt{\eta_T} \sigma_{t,p}(x)$. Term (1) can be bounded using the same steps as for standard GP-TS since the confidence bound $U_{t,p^*}(x^*)$ uses the true prior $p^*$. The key question for term (2) is whether a tight bound for $\sumt \sigma_{t,p_t}^2(x_t)$ can be obtained. \citet{ziomekTimevarying2025} provides the bound $\sum_{p \in P} \gamma_{N_{T}(p),p}$ as an intermediate step in the proof of Lemma 5.3 where $N_T(p)$ is the number of times prior $p$ is selected in total. Due to the nature of PE-GP-UCB (and similarly for PE-GP-TS) the only guarantee on $N_T(p)$ is that it is smaller than $T$, thus the bound $\sum_{p \in P} \gamma_{T, p}$ is used. However, we show in \cref{lem:infgain_selected_leq_avg_infgain} that a tighter bound can be obtained for HP-GP-TS, thereby improving the dependency upon the MIG compared to the bound of \citet{luSurrogate2023}. Under a Bayesian model, we show that $\E[N_T(p)] = P_1(p) T$ for HP-GP-TS and by a concavity argument we provide a bound in terms of $\bar{\gamma}_T (P_1) := \sump \Gamma_p(P_1(p)T)$ where $\Gamma_p(\cdot)$ is a continuous upper bound of $\gamma_{\cdot, p}$.

\begin{restatable}{lemma}{lemInfgainSelectedLeqAvgInfgain} \label{lem:infgain_selected_leq_avg_infgain}
    Let $C = 2 / \log(1 + \sigma^{-2})$, $N_T(p) = \sumt \indfcn{p_t = p}$ where $\1$ is the indicator function, and 
    $\Gamma_p: \R_{\geq 0} \mapsto \R_{\geq 0}$ be a concave function such that $ \Gamma_p(t) \geq \gamma_{t,p}$ for all $t, p \in [T] \times P$. Furthermore, let $\bar{\gamma}_T(P_1) := \sump \Gamma_p(P_1(p) T )$, then for HP-GP-TS, $\E[N_T(p)] = P_1(p)T$ and 
    \begin{equation}
        \E \Big[ \sumt \sigma_{t,p_t}^2(x_t) \Big] \leq C \bar{\gamma}_T(P_1).
    \end{equation}
\end{restatable}

To bound term (3), we define the excess reward function as 
$G_t(p) = \sum_{s=1}^{t-1} \indfcn{p_s=p}
\left(
        \mu_{s,p_s}(x_s) - \sqrt{\eta} \sigma_{s,p_s} - f(x_s) - \epsilon_s 
\right)$ for $\eta > 0$, similar to \citet{luSurrogate2023}. Then, we define the confidence set at time $t$ as 
$\C_t = \left\{ 
    p \in P: G_s(p) \leq \xi_s(p) \,\, \forall s \leq t  
\right\}$ where $\xi_t(p) = \sigma \sqrt{12 N_{t-1}(p) \log(T)}$ where $N_{t}(p) = \sum_{s=1}^{t} \indfcn{p_s = p}$ denotes how often the prior $p$ was selected up to and including time $t$. Unlike \citet{hongThompson2022,luSurrogate2023}, we impose a time-uniform requirement, i.e. a prior $p \in \C_t$ only if $p \in \C_s$ for all $s < t$, as we found their proofs uncompelling without this requirement, see \cref{rem:time-uniform-confset}. 
We show that $p^* \in \C_t$ with high probability in \cref{lem:true_prior_in_conf_set} and split term (3) into two new terms:
\begin{align} \label{eq:term3_decomposition}
    (3) = 
    \sumt \E \left[ (L_{t,p_t}(x_t) - f(x_t)) \indfcn{p_t \notin \C_t} \right ]
    + 
    \sumt \E \left[ (L_{t,p_t}(x_t) - f(x_t)) \indfcn{p_t \in \C_t} \right ].
\end{align}
Since $\C_t$ is defined to only consider the excess reward in the past, the right term can only be bounded up to the stopping time $\tau_p$ for each prior $p \in P$. The main hurdle is therefore to bound the expectation of the stopped excess reward for each prior $\E[ L_{\tau_p, p}(x_{\tau_p}) - f(x_{\tau_p}) - \epsilon_{\tau_p}]$. \citet{hongThompson2022,luSurrogate2023} provide incorrect bounds for this term in the linear and GP setting respectively, see \cref{sec:comparison}. We first note that the stopped value of this sequence can be bounded by the maximum over the same sequence and then provide a bound for $\E[\maxt (L_{t,p}(x_t) - f(x_t) - \epsilon_t)]$. For the left term, we know that $\E[\indfcn{p_t \notin \C_t}] = \Pb (p^* \notin \C_t) = \O(T^{-5})$ but the factor $L_{t,p_t}(x_t) - f(x_t)$ prevents direct application of this result. Again, the bounds provided by \citet{hongThompson2022,luSurrogate2023} do not hold. We make the observation that the two factors can be separated by the Cauchy-Schwarz inequality for expected values $(\E[XY] \leq \sqrt{\E[X^2]\E[Y^2]})$ and provide bounds for $\E[\mu^2_{t,p_t}(x_t)]$ and $\E[f(x_t)^2]$ in \cref{lem:expected_mean_bound,lem:bound_sup_f_squared}. Finally, we are ready to state our regret bound for HP-GP-TS.

\begin{restatable}{theorem}{thmHPGPTS} \label{thm:HPGPTS}
    Let $C = 2 / \log(1 + \sigma^{-2})$, $\mumax = \sup_{p,x \in P \times \Xc} |\mu_{1,p}(x)|$, $M = \E[ \sup_{x \in \Xc} |f(x)|]$, $M_p = \E\left[ \sup_{x \in \Xc} |f(x)| \big| p^* = p \right]$, $M_\Delta = \max_{p \in P} M_p - \min_{p \in P} M_p$, and $\bar{M} = M^2 + 1 + M_\Delta^2/4$.
    If $p^* \sim P_1$, $f \sim \GP(\mu_{1,p^*}, k_{1,p^*})$, $\beta_t = 2 \log (|\Xc| t^2 / \sqrt{2\pi})$, $\eta_T = 2 \log |\Xc| T^6$, then the Bayesian regret of HP-GP-TS is bounded by
    \begin{equation}
    \begin{aligned}[b]
        \BR(T) &\leq 
        \frac{\pi^2}{6} 
        + \sqrt{CT \bar{\gamma}_T(P_1)} (\sqrt{\beta_T} + \sqrt{\eta_T}) 
        \hspace{11.25em}(\text{Terms (1) and (2)})
        \\
        &\begin{rcases}
        + \frac{\sqrt{3}}{T} \Big( \sqrt{\sigma^{-2}( \bar{M} + 2M \mumax + \mumax^2 + \sigma^2)} + \sqrt{\bar{M}/T} \Big) 
        \\
         + \sigma \sqrt{14 T|P| \log T} + |P| \big( \sigma^{-1} \sqrt{T}(M + \mumax + \sigma) 
    + M 
    + \sigma \sqrt{2 \log T} \big)
    \end{rcases}
        \hspace{0em} (\text{Term (3)})
    \label{eq:hpgpts_regret_bound}
    \end{aligned}
    \end{equation}
\end{restatable}

Note that $M_p$ is the expected supremum of $|f(x)|$ given $p^*=p$ whereas $M$ is the expected supremum of $|f(x)|$ for the mixture $p^* \sim P_1$. Furthermore, $M_\Delta$ denotes the spread in expected supremums and $\bar{M}$ bounds the expectaction of the squared supremum $\sup_x f(x)^2$, see \cref{lem:bound_sup_f_squared}. 
Unlike PE-GP-TS, -UCB, and \citet{luSurrogate2023}, our regret bound of HP-GP-TS depends on the hyperprior-weighted MIG $\bar{\gamma}_T(P_1)$ rather than the worst case $|P| \hat{\gamma}_T$ which can impact the theoretical regret significantly if the complexity of the priors differ and the hyperprior is weighted towards simple priors. This is reasonable since the elimination methods assume arbitrary selection of $p^*$ as opposed to sampling from a hyperprior. The final term in \cref{eq:hpgpts_regret_bound} is $\O(|P| \sqrt{T}$) whereas the term for PE-GP-TS and -UCB that is linear in $|P|$ is constant w.r.t. $T$. 
In \cref{sec:experiments}, we empirically evaluate the dependency on $|P|$.

\subsection{Comparison to MixTS and EGP-TS} \label{sec:comparison}
\citet{hongThompson2022} study MixTS, a Thompson sampling algorithm that assumes the prior is a mixture distribution, for standard and linear bandits. 
For the linear setting with unbounded rewards, the proof requires conditioning on the linear parameter vector $\theta^*$ to lie close to its prior mean. Under this additional event $E_0$, the distribution of the true parameter vector $\theta^*$, the true prior $p^*$ and the optimal arm $x^*$ can shift. But, conditioned on the history $H_t$, MixTS is unaffected by conditioning on $E_0$ at time step $t$. Consequently, the sampled parameter vector $\theta_t$, the selected prior $p_t$ and the selected arm $x_t$ maintain the same distribution. However, \citet{hongThompson2022} use that $x^*, p^* | H_t, E_0 \eqd x_t,p_t | H_t, E_0$ without proof, invalidating Theorem 1 of \citet{hongThompson2022}. The event $E_0$ bounds the maximum per-round regret by a constant, enabling $L_{t,p}(x_{t}) - f(x_{t})$ to be conveniently bound by a constant for both terms in \cref{eq:term3_decomposition}. Unfortunately, the intermediate steps contain other issues that we discuss further in \cref{app:hongissues}.    
\citet{luSurrogate2023} study EGP-TS for sequential and parallel GP-bandit problems. Lemma 5 of \citet{luSurrogate2023} bounds $\E[\mu_{t,p_t}(x_t) - f(x_t)]$ by a constant $2B$ with an incorrect proof. Even assuming a correct proof, the lemma is applied incorrectly to claim that $\E[(\mu_{t,p_t}(x_t) - f(x_t)) \indfcn{p_t \notin \C_t}] \leq \E[2B \indfcn{p_t \notin \C_t}]$ and $\E[L_{\tau,p}(x_{\tau}) - f(x_{\tau}) - \epsilon_\tau] \leq 2B$.
Our proof avoids the issues in previous work by separating the event $\indfcn{p_t \notin \C_t}$ from the excess reward using the Cauchy-Schwarz inequality for expectations, bounding the stopped excess reward by the maximum excess reward, and bounding the expected values of $\mu_{t,p_t}(x_t)$, $\mu_{t,p}^2(x_t)$ and $f(x_t)^2$.

\begin{figure*}
    \centering
    \includegraphics[width=\linewidth,trim = {0 0.5em 0 0.3em}, clip]{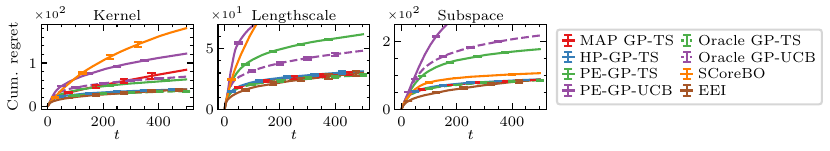}
    \caption{Cumulative regret for synthetic experiments with varying kernel, lengthscale and active subspace. The final regret for PE-GP-UCB is 114 and 389 in the lengthscale and subspace experiments, and 181 for SCoreBO in the lengthscale experiment. Errorbars correspond to $\pm1$ standard error.}
    \label{fig:regretsynthetic}
\end{figure*}

\section{Experiments} \label{sec:experiments}
In this section, we describe our experiments based on synthetic and real-world data.

\paragraph{Synthetic experiments}
We consider three synthetic setups with different choices of priors in $P$. For the first setup, the priors have one of the following kernels: i) RBF kernel, ii) the rational quadratic kernel with $\alpha = 0.5$, iii) Matérn kernel with $\nu=5/2$, iv) Matérn kernel with $\nu=3/2$, v) periodic kernel with period $\rho = 5$, vi) linear kernel with $v = 0.05^2$. 
For the second setup, 8 priors use the RBF kernel with different lengthscales equidistantly spaced between $1/2$ and $4$.
For the third setup, the total dimensions $d = 16$ but each of the 5 priors $p_i$ assumes $f(x)$ depends on $d_s = 4$ subdimensions. The 4 subdimensions are designed such that the priors are equally difficult to distinguish. 
All priors use the RBF kernel with lengthscale $\ell = 8$. For all three setups, the true prior $p^*$ is sampled uniformly from $P$, the noise variance $\sigma^2 = 0.25^2$, and the horizon $T = 500$. For the first two setups, 500 arms are equidistantly spaced in $[0, 20]$ and for the third 500 arms are sampled uniformly on $[0,20]^{16}$. 
All models are evaluated on 500 seeds on each setup. As baselines, we use PE-GP-UCB, SCoreBO \citep{hvarfner_self-correcting_2023}, fully Bayesian Expected Improvement (EEI) \citep{benassi_robust_2011} and Maximum A Posteriori (MAP) GP-TS. MAP GP-TS is identical to HP-GP-TS except for greedily selecting $p_t$ from the posterior: $p_t = \argmax_p P_{t}(p)$.\footnote{Note that since the hyperprior is uniform, MAP is equivalent to discrete maximum likelihood estimation.} 
In addition, we compare against the oracle variants of PE-GP-TS and PE-GP-UCB that are only given the true prior: $P_1 = \{ p^* \}$.

The cumulative regret for the three synthetic experiments is shown in \cref{fig:regretsynthetic} and the final regret is shown in \cref{tab:jointregret} in \cref{app:expresults}. Across all three experiments, we observe that HP-GP-TS and EEI has lower regret than the other methods and performs close to the oracle GP-TS. For the kernel and subspace experiments, PE-GP-TS has lower regret than the oracle GP-UCB. Hence, even if PE-GP-UCB was optimized to perform as well as the oracle, it would still not achieve the regret of the TS methods. MAP GP-TS has slightly higher regret than HP-GP-TS for the lengthscale and subspace experiments but has significantly higher regret and variance for the kernel experiment. The greedy selection of MAP (MLE) leads to under-exploration for MAP GP-TS in certain instances. SCoreBO has the highest regret in the kernel and lengthscale experiment but has more comparable performance in the subspace experiment. In \cref{fig:longregret} in \cref{app:expresults}, we report the regret for the two most competitive methods, HP-GP-TS and EEI, with an extended horizon $T = 1500$ where we observe that HP-GP-TS yields noticeably lower regret.

The maximum information gain, the number of priors remaining $|P_t|$ and the hyperposterior entropy for the kernel experiment is shown in \cref{fig:eliminationAndEntropy}. We note that $\bar{\gamma}_T(P_1)$ is significantly smaller than $|P|\hat{\gamma}_T$. The PE-methods eliminate at most one prior on average. 
In contrast, the final hyperposterior entropy across all algorithms is equivalent to 70-99\% of the probability mass being assigned to one prior showing that the hyperposterior adapts more effectively. Across the experiments, SCoreBO has the lowest hyperposterior entropy followed by HP-GP-TS and EEI has the highest (except for the kernel experiment), see \cref{fig:elimination,fig:entropy} in \cref{app:expresults}. Thus, HP-GP-TS has similar regret to EEI but lower hyperposterior entropy.

In \cref{fig:kernelconfusion}, we visualize how often the methods select the true prior $p^*$ (or kernel) in the kernel experiment as confusion matrices. PE-GP-UCB selects the Matérn-3/2 kernel more than 96\% of the rounds. The Matérn-3/2 kernel induces a distribution over functions that are less smooth compared to the other kernels and produces much wider confidence intervals outside the observed data leading to excessive optimistic exploration. PE-GP-TS also shows a bias towards the Matérn-3/2 kernel but does not select it as frequently as PE-GP-UCB -- demonstrating that one layer of optimism has been removed. The overall ``accuracy" of the selected priors, i.e. $\sumt \1\{ p_t = p^* \} / T$, for the elimination-based methods is around $17\%$ in the kernel experiment compared to $62.9\%$ and $63.2\%$ for MAP and HP-GP-TS respectively. For HP-GP-TS, we observe that it can easily identify the periodic and linear kernels. However, the RBF, Matérn and RQ kernels are often confused with each other. These kernels do not have as easily distinguishable characteristics and are likely to produce similar posteriors even with a small amount of data. See \cref{fig:confusionmatrices} in \cref{app:expresults} for confusion matrices in the lengthscale and subspace experiments.

\begin{figure}
    \centering
    \includegraphics[width=\linewidth,trim = {0 0.45em 0 0.35em}, clip]{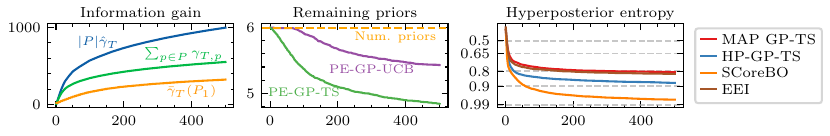}
    \caption{Analysis of the kernel experiment. Greedily maximal information gain (left). Mean number of priors remaining in $P_t$ over time for PE-GP-UCB and -TS (middle). Entropy in the hyperposterior $P_t$ over time for HP- and MAP GP-TS (right). The dashed reference lines correspond to entropies of discrete distributions with prob. $q$ on one choice and prob. $\frac{1-q}{|P|-1}$ on the other $|P|-1$ choices.}
    \label{fig:eliminationAndEntropy}
\end{figure}
\begin{figure}
    \centering
    \includegraphics[width=\linewidth, trim = {0 0.35em 0 0.35em}, clip]{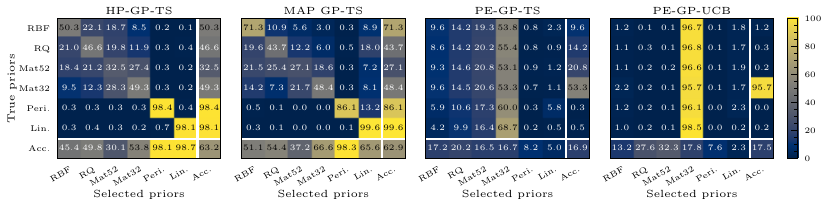}
    \caption{Confusion matrices for the true prior $p^*$ and the selected priors $p_t$ for the kernel experiment. 
    }
    \label{fig:kernelconfusion}
\end{figure}

\begin{figure}
    \centering
    \includegraphics[width=\linewidth, trim = {0 0.3em 0 0.3em}, clip]{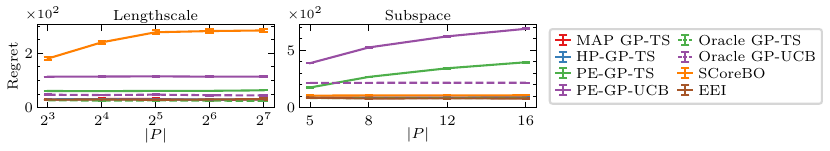}
    \caption{Total regret for the lengthscale and subspace experiments as $|P|$ increases.}
    \label{fig:regretscaling}
\end{figure}

\begin{figure*}
    \centering
    \includegraphics[width=\linewidth,trim = {0 0.5em 0 0.30em}, clip]{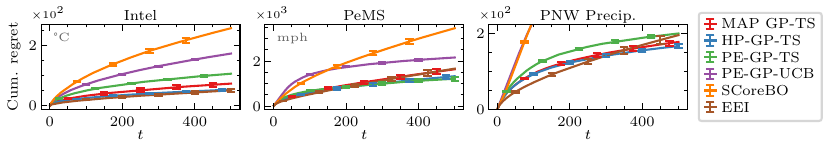}
    \caption{Cumulative regret on the real-world data experiments. Errorbars correspond to $\pm1$ standard error. The average final regret for SCoreBO and PE-GP-UCB is 861 and 506 on PNW.} 
    \label{fig:realworldregret}
\end{figure*}

\paragraph{Scaling $|P|$} We perform two experiments to understand how the regret of our algorithms scale with the number of priors. In both experiments, the average difficulty of the problem is kept constant such that the regret of the oracle models is constant. In the first experiment, we increase the discretization of the lengthscale values. The lengthscales are equidistantly spaced in $[0.5, 4]$ with $|P| \in \{8, 16, 32, 64, 128 \}$. 
As $|P|$ increases, the difference between similar priors is reduced.
In the second experiment, we increase the number of priors in the subspace experiment from 5 up to 16. Each prior can share at most 3 out of 4 dimensions with other priors which ensures the priors remain meaningfully different. The total regret as the number of priors increases is shown in \cref{fig:regretscaling}. For the lengthscale experiment, increasing the number of priors above 8 does not affect the regret for any algorithm, likely due to the increased redundancy in the priors. The one exception is SCoreBO, whose regret increases as $|P|$ is increased to 32 but levels off beyond that. In the subspace experiment, the regret of the prior elimination algorithms scales approximately as $\sqrt{|P|}$ whilst MAP- and HP-GP-TS are consistently close to the constant regret of the oracle. The regret of EEI drops initially but is otherwise constant and SCoreBO also has constant regret.

\paragraph{Real-world data}
We perform three experiments with real-world data from the Intel Berkeley dataset \citep{madden2004intel}, California Performance Measurement System (PeMS) \citep{chenFreeway2001, pems} and Pacific Northwest (PNW) daily precipitation dataset \citep{widmann_50_1999,widmann_validation_2000}. Each dataset contains measurements from a set of sensors over time. We split each dataset into a training and test set where the test set contains the last third of the data. Hence, the distribution of the test data may have shifted from the training data allowing us to test a realistic setting where the true prior is unknown but we have a set of reasonable priors. Each training set is further split into separate buckets which we use to estimate the empirical mean and covariance of the priors. See \cref{app:expdetails} for more details.

The cumulative regret for the experiments with real-world data is presented in \cref{fig:realworldregret}. Across these experiments, HP-GP-TS has either the lowest regret or is within 1 standard error of the algorithm with the lowest regret. SCoreBO has significantly higher regret than all other methods. Notably for the PeMS data, PE-GP-TS has the lowest average regret whereas MAP GP-TS and EEI perform worse compared to the other experiments. 

The number of priors remaining in $|P_t|$ and the hyperposterior entropy for the real-world data experiments is shown in \cref{fig:elimination,fig:entropy} in \cref{app:expresults}.
Similar to the synthetic experiments, on average, the prior elimination methods eliminate less than 1 prior at best and no priors (across all 500 seeds) at worst. In contrast, the hyperposterior of HP-GP-TS concentrates to the equivalent of 60-80\% of the probability mass to one prior. The relative standing in terms of reduced hyperposterior uncertainty between SCoreBO, HP-GP-TS and EEI remains consistent across all the experiments. 
SCoreBO reduces the hyperposterior uncertainty the most at the cost of significantly higher regret whereas HP-GP-TS provides a better balance between low regret and low hyperposterior uncertainty.

\section{Conclusion}
In this paper, we have studied two algorithms for adaptive prior selection 
and regret minimization 
in GP bandits based on GP-TS. 
We have analyzed the algorithms theoretically, corrected and improved upon previous work, and experimentally evaluated both algorithms on synthetic and real-world data. 
We find that lowering the amount of optimistic exploration leads the algorithms to obtain lower or comparable regret than previous work.

\section*{Acknowledgements}
The work of Jack Sandberg and Morteza Haghir Chehreghani was partially supported by the Wallenberg AI, Autonomous Systems and Software Program (WASP) funded by the Knut and Alice Wallenberg Foundation. The computations were enabled by resources provided by the National Academic Infrastructure for Supercomputing in Sweden (NAISS), partially funded by the Swedish Research Council through grant agreement no. 2022-06725.

\bibliographystyle{bibstyle}
\bibliography{references}


\appendix

\newpage
\section{Extended discussion of related work} \label{app:extended_related_work}
Plenty of previous work has proposed fully Bayesian approaches that integrate the acquisition function over the hyperposterior \citep{osborne_gaussian_2009,benassi_robust_2011, snoek_practical_2012, hernandez-lobato_predictive_2014, wang_max-value_2017,de_ath_how_2021}. A difficulty with such approaches is that to compute the expected acquisition function they must perform costly MCMC sampling over the hyperposterior. In contrast, HP-GP-TS optimizes a single hyperposterior sample instead of computing expected values over the hyperposterior.
\citet{hvarfner_self-correcting_2023} proposed Self-Correcting Bayesian Optimization (SCoreBO) whose objective function balances reducing the uncertainty of $(x^*, f^*)$ and reducing the uncertainty of the true prior $p^*$. Notably, SCoreBO explicitly tries to identify the prior rather than integrating out the uncertainty of the prior.  

\citet{wangTheoretical2014} first derived regret bounds for GP bandits with unknown lengthscale for the Expected Improvement algorithm \citep{mockusBayesian1975}. 
However, the proposed algorithm requires a lower bound on the lengthscale and the regret bound depends on the worst-case MIG.
Later work by \citet{berkenkampNoRegret2019} introduced Adaptive GP-UCB (A-GP-UCB) that continually lowers the lengthscale parameter. Given a sufficiently small lengthscale, the function $f$ lies within the reproducing kernel Hilbert space (RKHS) and the regular GP-UCB theory can be applied. However, A-GP-UCB lacks a stopping mechanism and will overexplore as the lengthscale continues to shrink. 
Recent work by \citet{ziomekTimevarying2025} introduced Prior-Elimination GP-UCB (PE-GP-UCB) for time-varying GP-bandits with unknown prior. Unlike the work before, the regret bound of PE-GP-UCB holds for arbitrary types of hyperparameters in the GP prior. PE-GP-UCB is doubly optimistic and selects the prior {\it and} arm with the highest upper confidence bound. PE-GP-UCB tracks the cumulative prediction error made by the selected priors and eliminates priors that exceed a threshold level.

Other works have introduced regret balancing algorithms that maintain a set of base learning algorithms and balance their selection frequency to achieve close to optimal regret \citep{abbasi-yadkoriRegret2020, pacchianoRegret2020}. \citet{ziomekBayesian2024} built on this idea and introduced length-scale balancing GP-UCB which can adaptively explore smaller lengthscales but can return to longer ones, unlike A-GP-UCB.

In addition to \citet{hongThompson2022,luSurrogate2023}, another line of work has studied Thompson sampling in standard and linear bandits with unknown prior distribution \citep{kvetonMetaThompson2021, basuNo2021, hongHierarchical2022,liModified2024}. In their setting (meta or hierarchical bandits), the agent plays multiple bandit instances, either simultaneously or sequentially. The unknown means are sampled from the same (unknown) prior and by gathering knowledge across instances, the agent can solve later instances more efficiently once it has identified the prior. In contrast, in this paper we consider the setting where the agent can only access information from the instance it is facing.

\section{Proofs} \label{app:allproofs}
In the following section, we state and prove the results shown in the main text.

\subsection{PE-GP-TS} \label{sec:proofsPEGPTS}
First, we state and prove concentration inequalities for $f(x)$ and $\tilde{f}_{t,p}(x)$. \cref{lem:concentrationintervals} is based on Lemma 5.1 of \citet{srinivasInformationTheoretic2012} but adapted to TS by specifying that it holds for any sequence of $x_1, \ldots, x_T$, as discussed by \citet{russoLearningOptimizePosterior2014}. Additionally, we add \cref{eq:sampleconfinv} which can be shown through the same steps and an additional union bound over $P$.

\begin{restatable}{lemma}{lemConcentrationIntervals} \label{lem:concentrationintervals}
    If $f(x) \sim \GP(\mu_{1, p^*}, k_{1, p^*})$ and $\beta_t = 2 \log \left ( \frac{|\Xc| |P| \pi^2 t^2}{3 \delta}\right)$. Then, with probability at least $1 - \delta$, the following holds for all $t, x, p \in [T] \times \Xc \times P$:
    \begin{align}
        | f(x) - \mu_{t, p^*}(x) | &\leq \sqrt{\beta_t} \sigma_{t, p^*}(x),
        \label{eq:postconfinv} \\
        | \tilde{f}_{t, p} (x) - \mu_{t, p} (x) | &\leq \sqrt{\beta_t} \sigma_{t, p}(x). \label{eq:sampleconfinv}
    \end{align}
\end{restatable}
\begin{proof}
    Follows by the same steps as Lemma 5.1 of Srinivas except we condition on the complete history $H_t$ instead of only $\y_{1:t-1}$. Additionally, for \cref{eq:sampleconfinv} we must take an additional union bound over $p \in P$.

    Fix $t, x, p \in [T] \times \Xc \times P$. Given the history $H_t$, $\tilde{f}_{t, p}(x) \sim \N(\mu_{t,p}(x), \sigma^2_{t,p}(x))$. 
    Using that $\mathbb{P}( Z > c ) \leq 1/2 e^{-c^2/2}$ for $Z \sim \N(0,1)$, we get that 
    \begin{align}
        \mathbb{P}\left( \left| \frac{\tilde{f}_{t, p}(x) - \mu_{t, p}(x) }{\sigma_{t, p} (x)} \right|  > \sqrt{\beta_t} \right) &\leq \exp( - \beta_t / 2 ) \\
        &= \frac{3 \delta }{ |\Xc| |P| \pi^2 t^2}
    \end{align}
    Note that $\sum_{t\geq1} \frac{1}{t^2} = \frac{\pi^2}{6}$. By taking the union bound over $\Xc$, $P$ and $t \geq 1$, \cref{eq:sampleconfinv} holds w.p. at least $1 - \delta/2$. By the same reasoning and skipping the union bound over $P$, \cref{eq:postconfinv} holds w.p. at least $1-\delta/2$. Thus, both events hold w.p. at least $1 - \delta$.
\end{proof}

Next, we state three lemmas from \citet{ziomekTimevarying2025} that are used in the proof of our regret bound.

\begin{lemma}{(Lemma 5.1 of \citet{ziomekTimevarying2025})} \label{lem:noiseconc}
    If $\xi_t = 2 \sigma^2 \log  \left( \frac{|P| \pi^2 t^2}{6 \delta} \right)$, then the following holds with probability at least $1 - \delta$:
    \begin{equation}
        \left| \sum_{i \in S_{t, p}} \epsilon_i \right| \leq \sqrt{\xi_t |S_{t, p}|} \quad \forall t, p \in [T] \times P.
    \end{equation}
\end{lemma}
\begin{lemma}{(Lemma 5.2 of \citet{ziomekTimevarying2025})} \label{lem:fbound}
    Let $B_{p^*} = \beta_1 + \sup_{x \in \Xc} |\mu_{1, p^*}(x)|$, then if $\mu_{1,p^*}$ and $k_{1,p^*}$ satisfy $|\mu_{1,p^*} (\cdot)| < \infty$ and $k_{1,p^*}(\cdot, \cdot) \leq 1$ and \cref{lem:concentrationintervals} holds, then 
    \begin{equation}
        \sup_{x \in \Xc} |f(x)| \leq B_{p^*}.
    \end{equation}
\end{lemma}
\begin{lemma}{(Lemma 5.3 of \citet{ziomekTimevarying2025})} \label{lem:infgainbound}
    For $C = 2 / \log(1 + \sigma^{-2})$, $\sum_{t \notin \mathcal{C}} \sqrt{\beta_t} \sigma_{t, p_t} (x_t) \leq \sqrt{C T \beta_T \hat{\gamma}_{T} |P|}$ where $\beta_T = \max_{p \in P} \beta_T$ and $\hat{\gamma}_T = \max_{p \in P} {\gamma}_{T,p}$.
\end{lemma}

\begin{lemma} \label{lem:nevereliminated}
    If the events of \cref{lem:concentrationintervals,lem:noiseconc} hold, then PE-GP-TS never eliminates the true prior $p^*$.
\end{lemma}
\begin{proof}
    For any $t \in [T]$,
    \begin{align}
        \left| \sum_{i \in S_{t, p^*}} \eta_i \right| &= \left| \sum_{i \in S_{t, p^*}} (y_i - f(x_i) + f(x_i) - \mu_{i, p^*}(x_i) \right| \\
        &\leq \left| \sum_{i \in S_{t, p^*}} \epsilon_i \right| + \sum_{i \in S_{t, p^*}} \left| f(x_i) - \mu_{i, p^*}(x_i) \right| 
        & (\text{Triangle ineq.})\\
        &\leq \sqrt{\xi_t |S_{t, p^*}| } + \sum_{i \in S_{t, p^*}} \sqrt{\beta_i} \sigma_{i, p^*}(x_i).
        &(\text{\cref{lem:noiseconc,lem:concentrationintervals})}
    \end{align}
    Therefore, the elimination criteria on \cref{line:elimincriteria} in \cref{alg:PEGPTS}, $|\sum_{i \in S_{t,p_t} \eta_i}| > V_t$, always evaluates to \texttt{false} for $p_t = p^*$.
\end{proof}

Then, we state and prove the new instantaneous regret bound for PE-GP-TS. 

\lemSimpleRegret*
\begin{proof}
    First, we upper bound $f(x^*)$ as follows
    \begin{align}
        f(x^*) &\leq \mu_{t, p^*}(x^*) + \sqrt{\beta_t} \sigma_{t, p^*} (x^*)
        &(\text{\cref{eq:postconfinv}}) \\
        &\leq \tilde{f}_{t, p^*}(x^*) + 2 \sqrt{\beta_t} \sigma_{t, p^*}(x^*) 
        &(\text{\cref{eq:sampleconfinv}}) \\
        &\leq \tilde{f}_{t, p_t}(x_t) + 2 \sqrt{\beta_t} \sigma_{t, p^*}(x^*).
        &(\text{TS selection rule and \cref{lem:nevereliminated}}) \label{eq:simpleupper}
    \end{align}
    For the final step, note that $p^* \in P_t$ by \cref{lem:nevereliminated}.
    Then, we lower bound $f(x_t)$
    \begin{align}
        f(x_t) &= \mu_{t, p_t}(x_t) + \eta_t - \epsilon_t 
        &(\text{Def. of } \eta_t) \\
        &\geq \tilde{f}_{t, p_t} (x_t) - \sqrt{\beta_t} \sigma_{t, p_t} (x_t) + \eta_t - \epsilon_t.
        &(\text{\cref{eq:sampleconfinv}}) \label{eq:simplelower}
    \end{align}
    Combining, \cref{eq:simpleupper,eq:simplelower} we obtain 
    \begin{align}
        f(x^*) - f(x_t) \leq 2 \sqrt{\beta_t} \sigma_{t, p^*} (x^*) + \sqrt{\beta_t} \sigma_{t, p_t}(x_t) - \eta_t + \epsilon_t. \label{eq:simpleregretbound}
    \end{align}
\end{proof}

Finally, we state and prove the cumulative regret bound for PE-GP-TS.

\thmPEGPTSregretbound*
\begin{proof}    
    To establish a bound on the cumulative regret, we separate out the rounds where priors are eliminated. Hence, define the set of critical iterations as
    \begin{equation}
        \mathcal{C} = \left\{ t \in [T] : \left| \sum_{i \in S_{t, p_t}} \eta_i \right| > \sqrt{\xi_t S_{t, p_t}} + \sum_{i \in S_{t, p_t}} \sqrt{\beta_i} \sigma_{i, p_t}(x_i) \right\}.
    \end{equation}
    Note that $|\C| \leq |P|$.
    Using \cref{lem:fbound,eq:simpleregretbound}, we can bound the cumulative regret as follows:
    \begin{align}
        R(T) &= \sum_{t \in \mathcal{C}} f(x^*) - f(x_t) + \sum_{t \notin \mathcal{C}} f(x^*) - f(x_t) \\
        &\leq 2 |P| B_{p^*}  + \sum_{t \notin \mathcal{C}} 2 \sqrt{\beta_t} \sigma_{t, p^*}(x^*) 
        + \sum_{t \notin \mathcal{C}} \sqrt{\beta_t} \sigma_{t, p_t} (x_t) + \sum_{p \in P} \sum_{t \in S_{T, p} \setminus \mathcal{C}} (\epsilon_t - \eta_t).
    \end{align}
    where $B_{p^*} := \beta_1 + \sup_{x \in \Xc} | \mu_{1, p^*} (x)|$.
    If $t \notin \mathcal{C}$, \cref{line:elimincriteria} in \cref{alg:PEGPTS} evaluates to {\tt false} and hence
    \begin{equation}
        \sum_{p \in P} \sum_{t \in S_{T, p} \setminus \mathcal{C}} -\eta_t \leq \sum_{p \in P} \sqrt{\xi_T |S_{T, p}|} + \sum_{p \in P} \sum_{t \in S_{T, p} \setminus \mathcal{C}} \sqrt{\beta_t} \sigma_{t, p}(x_t). 
    \end{equation}
    Additionally, using \cref{lem:noiseconc}, we can bound the Gaussian noise:
    \begin{align}
        \sum_{p \in P} \sum_{t \in S_{T, p} \setminus \mathcal{C}} \epsilon_t 
        &\leq \sum_{p \in P} \left| \sum_{t \in S_{T, p} \setminus \mathcal{C}} \epsilon_t \right| 
        \\
        &\leq \sum_{p \in P} \sqrt{\xi_T |S_{T, p} \setminus \C|} 
        &(\text{\cref{lem:noiseconc}})\\
        &\leq \sum_{p \in P} \sqrt{\xi_T |S_{T, p}|}  \\
        &\leq \sqrt{\xi_T |P| T}. 
        & (\text{Cauchy-Schwarz})
    \end{align}
    Combining the above, the cumulative regret is bounded by
    \begin{equation}
        R(T) \leq 2 |P| B_{p^*} + 2 \sqrt{\xi_T |P| T} 
        + 2 \sum_{t \notin \mathcal{C}} \sqrt{\beta_t} \sigma_{t, p^*}(x^*)
        + 2 \sum_{t \notin \mathcal{C}} \sqrt{\beta_t} \sigma_{t, p_t}(x_t).
    \end{equation}
    Finally, applying \cref{lem:infgainbound}, we obtain the result
    \begin{equation}
        R(T) \leq 2 |P| B_{p^*} + 2 \sqrt{\xi_T |P| T} + 2 \sqrt{T \beta_T \sumt \sigma^2_{t,p^*}(x^*)} +2 \sqrt{C T \beta_T \hat{\gamma}_{T} |P|}.
    \end{equation}
\end{proof}

\subsection{HP-GP-TS} \label{sec:proofsHPGPTS}
In this section, we state and prove our regret bound for HP-GP-TS. We begin by proving \cref{lem:infgain_selected_leq_avg_infgain}.

\lemInfgainSelectedLeqAvgInfgain*
\begin{remark}
    For the RBF and Matérn kernels, the known upper bounds for the maximum information gain are concave \citep{srinivasInformationTheoretic2012,vakiliInformationGainRegret2021}, thereby satisfying the conditions of \cref{lem:infgain_selected_leq_avg_infgain}. 
\end{remark}
\begin{proof}
    We begin by showing that $\E[N_T(p)] = P_1(p)T$.
    \begin{align}
        \E[N_T(p)] &= \E \Big[ \sumt \indfcn{p_t = p} \Big] \\
        &= \sumt \E_{H_t} \left[ \E \left[ \indfcn{p_t = p} \big| H_t \right] \right]
        && (\text{Tower rule}) \\
        &= \sumt \E_{H_t} \left[ \E \left[ \indfcn{p^* = p} \big| H_t \right] \right]
        && (p_t | H_t \eqd p^* | H_t) \\
        &= \sumt \E \left[ \indfcn{p^* = p} \right] \\
        &= \sumt P_1(p) = P_1(p) T. 
        && \left( P_1(p) = \Pb(p^* = p) \right) \label{eq:expected_num_selections}
    \end{align}
    Then, by the intermediate steps of Lemma 5.3 of \citet{ziomekTimevarying2025} $\sumt \sigma^2_{t,p_t}(x_t) \leq C \sump \gamma_{N_T(p), p}$. We include the proof here for completeness and introduce some helpful notation. 
    Let $A$ be a multiset over $\Xc$ s.t. $|A| < \infty$, we define 
    \begin{equation}
        \sigma_{A,p}^2(x) = k_{1,p}(x,x) - \k_{A,p}(x)^T (\K_{A,p} + \sigma^2 I)^{-1} \k_{A,p}(x),
    \end{equation}
    where $\K_{A,p} = [k_{1,p}(x, x')]_{x,x' \in A}$ and $\k_{A,p} = [k_{1,p}(x,x')]_{x' \in A}$ with elements repeated by their multiplicity in $A$. Then, let $A_{t,p} = \{ x_i : i \in [t-1], p_i = p \}$ be the multiset of arms queried whilst selecting prior $p$. For any two multisets $S,S'$ such that $S \subseteq S'$, we have that $\sigma_{S',p}^2(x) \leq \sigma_{S,p}^2(x)$ for all $x \in \Xc$. Since $A_{t,p}$ is a subset of the arms collected over the history $H_t$, we have that $\sigma_{t,p_t}^2(x_t) \leq \sigma_{A_{t,p},p}^2(x_t)$. By the proof of Lemma 5.4 of \citet{srinivasInformationTheoretic2012}, $\sigma_{A_{t,p},p}^2(x_t) \leq C \log(1 + \sigma^{-2} \sigma_{A_{t,p},p}^2(x_t))$ for $C = 2 /\log(1 + \sigma^{-2})$. By reorganizing the sum over $t \in [T]$ into a sum over $p \in P$, and applying Lemma 5.3 of \citet{srinivasInformationTheoretic2012} yields that
    \begin{equation}
        \sumt \sigma_{t,p_t}^2(x_t) \leq \sum_{p \in P} \sum_{t \in [T] : p_t = p} \sigma_{A_{t,p},p}^2(x_t) \leq C \sump \gamma_{|A_{T,p}|,p} = C \sump \gamma_{N_{T}(p),p}. \label{eq:sum_post_var_leq_sum_inf_gain}
    \end{equation}
    Finally, we combine the results above through a concavity argument.
    \begin{align}
        \E \left[ \sumt \sigma_{t,p_t}^2(x_t) \right]
        &\leq C \sum_{p \in P} \E[\gamma_{N_T(p),p}]
        &&(\text{\cref{eq:sum_post_var_leq_sum_inf_gain}})\\
        &\leq C \sum_{p \in P} \E[\Gamma_p(N_T(p))]
        &&(\gamma_{t,p} \leq \Gamma_{p}(t), \forall t,p \in [T] \times P) \\
        &\leq C \sum_{p \in P} \Gamma_p (\E[N_T(p)] ) 
        &&(\text{$\Gamma_{p}(t)$ concave, Jensen's ineq.})\\
        &\leq C \sum_{p \in P} \Gamma_p ( P_1(p) T ).
        &&(\text{\cref{eq:expected_num_selections}})
    \end{align}
\end{proof}

Next, we prove that the true prior is in the confidence set with high probability. Recall that we define the excess reward for prior $p$ at time $t$ as
\begin{equation}
    G_t(p) = \sum_{s=1}^{t-1} \indfcn{p_s=p}\left(
        \mu_{s,p_s}(x_s) - \sqrt{\eta_T} \sigma_{s,p_s} - f(x_s) - \epsilon_s 
    \right) \label{eq:excess_reward}
\end{equation}
where $\eta_T = 2\log |\Xc| T^6$.
Let $\xi_t(p) = \sigma \sqrt{14 N_{t-1}(p) \log(T)}$ where $N_{t}(p) = \sum_{s=1}^{t} \indfcn{p_s = p}$ denotes how often the prior $p$ was selected up to and including time $t$. Then, we define the confidence set at time $t$ as 
\begin{equation} 
    \C_t = \left\{ 
        p \in P: G_\tau(p) \leq \xi_\tau(p) \,\, \forall \tau \leq t  
    \right\}. \label{eq:confidence_set}
\end{equation}
For notational convenience, we consider the history $H_t = (p_i, x_i, y_i)_{i=1}^{t-1}$ with the selected priors $(p_i)_{i=1}^t$ augmented such that $\C_t$ and $(N_{t-1}(p))_{p \in P}$ are deterministic conditioned on $H_t$. 

\begin{restatable}{lemma}{lemTruePriorInConfSet} \label{lem:true_prior_in_conf_set}
    For any $t \in [T]$, $\Pb(p^* \notin \C_t) \leq \frac{3}{T^5}$.
\end{restatable}

\begin{proof}
    Note that $\C_t$ is monotonically decreasing due to the time-uniform definition of $\C_t$, i.e. $\C_s \supseteq \C_t$ for any $s < t$. Thus, $\Pb( p^* \notin \C_t) \leq \Pb( p^* \notin \C_T)$ and we focus on bounding $\Pb( p^* \in \C_T)$.

    Let $E = \cap_{t=1}^{T-1} E_t$ where $E_t = \{|f(x) - \mu_{t,p^*}(x)| \leq \sqrt{\eta_T} \sigma_{t,p^*}(x), \forall x \in \Xc \}$. Then, by the law of total probability
    \begin{align}
        \Pb( p^* \notin \C_T) &= \underbrace{\Pb( p^* \notin \C_T | E^c)}_{\leq 1} \underbrace{\Pb(E^c)}_{\leq 1/T^5} + \underbrace{\Pb( p^* \notin \C_T | E) \Pb(E)}_{\leq 2 / T^5}
    \end{align}
    where the bounds for $\Pb(E^c)$ and $\Pb(p^* \notin \C_T | E) \Pb(E)$ are shown below.
    \begin{align}
        \Pb(E^c) &= \Pb\left( \exists t \in [T-1], x \in \Xc: |f(x) - \mu_{t,p^*}(x)| > \sqrt{\eta_T} \sigma_{t,p^*}(x) \right) \\
        &\leq \sum_{t \in [T-1]} \sum_{x \in \Xc} \Pb\left( |f(x) - \mu_{t,p^*}(x)| > \sqrt{\eta_T} \sigma_{t,p^*}(x) \right) \\
        &\leq \sum_{t \in [T-1]} \sum_{x \in \Xc} \E_{H_t, p^*}\left[ \Pb\left( \frac{|f(x) - \mu_{t,p^*}(x)|}{\sigma_{t,p^*}(x) } > \sqrt{\eta_T} \bigg| H_t, p^* = p \right) \right] \\
        &\leq \sum_{t \in [T-1]} \sum_{x \in \Xc} \E_{H_t,p^*}\left[ \exp\left( -\frac{\eta_T }{ 2} \right) \right] 
         \hspace{6.5em} \left( \parbox{11em}{\centering $\Pb(|r| > \sqrt{c}) \leq \exp(-c/2)$ \\ for $r \sim \N(0,1)$, $c \geq 0$} \right) \\
        &= \sum_{t \in [T-1]} \sum_{x \in \Xc}  \E_{H_t,p^*}\left[ \frac{1}{|\Xc| T^6} \right]
        \hspace{12em} \left( \eta_T = 2 \log (|\Xc| T^6)\right) \\
        & = \frac{1}{T^5}.
    \end{align}
    
    Next, we bound the right term $\Pb(p^* \notin \C_t | E)$. Recall that $p^* \notin \C_T$ is equivalent to $\exists t \in [T]$ such that $G_t(p^*) > \xi_t(p^*)$. Hence,
    \begin{align}
        \Pb(p^* \notin \C_T | E) &= \Pb( \exists t \in [T]: G_t(p^*) > \xi_t(p^*) | E) \\
        &\leq \sumt \Pb( G_t(p^*) > \xi_t(p^*) | E) 
        & \left( \text{Union bound} \right)
        \\ &\begin{multlined}[b]
        =\sumt \Pb\Bigg( \sum_{s=1}^{t-1} \indfcn{p_s = p^*} \Big( \mu_{s,p^*}(x_s) - \sqrt{\eta_T} \sigma_{t,p^*}(x_s) 
        \\
        - f(x_s) - \epsilon_s \Big) > \xi_t(p^*) \Big| E \Bigg).
        \end{multlined}
    \end{align}
    Given $E$, $\mu_{s,p^*}(x_s) - \sqrt{\eta_T} \sigma_{s,p^*}(x_s) - f(x_s) \leq 0$, $\forall s \in [T-1]$ and therefore
    \begin{align} \label{eq:pstar_not_in_Ct_given_E}
        \Pb(p^* \notin \C_T | E) \Pb(E) 
        &\leq \sum_{t \in [T]} \Pb\left( \sum_{s=1}^{t-1} \indfcn{p_s = p^*} (-\epsilon_s) >  \xi_t(p^*) \Big| E \right)\Pb(E) \\
        &\leq \sum_{t \in [T]} \Pb\left( \left| \sum_{s=1}^{t-1} \indfcn{p_s = p^*} (-\epsilon_s) \right| >  \xi_t(p^*) \Big| E \right)\Pb(E) \\
        &= \sum_{t \in [T]} \Pb\left( \left| \sum_{s=1}^{t-1} \indfcn{p_s = p^*} (-\epsilon_s) \right| >  \xi_t(p^*) , E \right) 
        \\
        &\leq \sum_{t \in [T]} \Pb\left( \left| \sum_{s=1}^{t-1} \indfcn{p_s = p^*} (-\epsilon_s) \right| >  \xi_t(p^*) \right) \\
        &\begin{aligned}[b] = \sum_{t \in [T]} \sum_{p \in P} &\Pb\left( 
            \left| \sum_{s=1}^{t-1} \indfcn{p_s = p} \epsilon_s \right| > \sigma \sqrt{14 N_{t-1}(p) \log T}
            \,\Bigg|\, p^* =p 
        \right) \\ &\quad \cdot \Pb(p^*=p) \end{aligned} \label{eq:selfnorm_conc_used}\\
        &\leq \sum_{t \in [T]} \sum_{p \in P} \frac{2}{T^6} \Pb(p^*=p) \hspace{12em} (\text{\cref{lem:selfnorm_conc_ineq}}) \label{eq:selfnorm_conc_used2}\\
        &\leq \frac{2}{T^5}
    \end{align}
\end{proof}

Next, we prove the self-normalizing concentration inequality for the sum of Gaussian noises as pulled by each prior that we used in \cref{eq:selfnorm_conc_used}. 

\begin{lemma} \label{lem:selfnorm_conc_ineq}
    Let $S_{t,p} = \sum_{s=1}^{t-1} \indfcn{p_s=p} \epsilon_s$ and $\alpha > 0$, then 
    \begin{equation}
        \Pb \left( |S_{t,p}| > \sigma \sqrt{2(\alpha + 1) N_{t-1}(p) \log(T)}  \,\Big|\, p^* = p \right) \leq \frac{2}{T^{\alpha}}, \quad \forall p \in P.
    \end{equation}
\end{lemma}
\begin{remark}
Note that similar results have been shown by \citet[Proof of Lemma 3]{hongThompson2022}, \citet[Lemma 4]{luSurrogate2023}, and \cite[Lemma 5.1]{ziomekTimevarying2025}. We found the arguments in the proofs of \citet{hongThompson2022,luSurrogate2023} uncompelling due to their brevity. Whilst the proof of \citet{ziomekTimevarying2025} is clearer, we provide a proof using a martingale technique as a complement.
\end{remark}
\begin{proof}
    Fix $t \in [T]$ and $p^* = p$, for the remainder of this proof all probabilities condition on $p^*=p$. We begin by defining the event 
    \begin{align}
        \mathcal{F} &:= \left\{ S_{t,p} > \sigma \sqrt{2(\alpha+1) N_{t-1}(p) \log(T)} \right\} \\
        &= \bigcup_{k=1}^{t-1} \underbrace{\left\{ S_{t,p} > \sigma \sqrt{2 (\alpha+1) k \log(T)} \,\cap\, N_{t-1}(p) = k \right\}}_{\mathcal{F}_k :=}.
    \end{align}
    To bound the probability of the events $\mathcal{F}_k$, we introduce a martingale $M_t(\lambda)$ for $\lambda > 0$ into $\mathcal{F}_k$ as follows:

    \begin{align}
        \mathcal{F}_k &= \left\{ \lambda S_{t,p} > \lambda \sigma \sqrt{2(\alpha+1) k \log(T)} \,\cap\, N_{t-1}(p) = k \right\} \hspace{10em} (\lambda > 0) \\
        &= \left\{ 
        \lambda S_{t,p} - \frac{\lambda^2\sigma^2}{2}N_{t-1}(p)
        > \lambda \sigma \sqrt{2(\alpha+1) k \log(T)} - \frac{\lambda^2\sigma^2}{2}k 
        \,\cap\, N_{t-1}(p) = k 
        \right\} \\
        &\begin{multlined}[b]= \Bigg\{ 
        \underbrace{
            \exp \left( \lambda S_{t,p} - \frac{\lambda^2\sigma^2}{2}N_{t-1}(p) \right)
        }_{M_t(\lambda) :=}
        > \exp \left( \lambda \sigma \sqrt{2(\alpha+1) k \log(T)} - \frac{\lambda^2\sigma^2}{2} k \right)
        \\
        \,\cap\, N_{t-1}(p) = k 
        \Bigg\}.
        \end{multlined}
    \end{align}
    To tighten the bound of the probability of $\mathcal{F}_k$, we select $\lambda_k = \sqrt{\frac{2 (\alpha+1) \log T}{\sigma^2 k}}$, yielding:
    \begin{align}
        \mathcal{F}_k &= 
        \left\{ 
        M_t(\lambda_k)
        > \exp \left( (\alpha+1) \log T \right)
        \,\cap\, N_{t-1}(p) = k 
        \right\} \\
        &\subseteq \left \{ M_t\left(\sqrt{\frac{2 (\alpha+1) \log T}{\sigma^2 k}}\right) \geq T^{\alpha+1} \right\}.
    \end{align}
    Since $M_t(\lambda) \geq 0$, by Markov's inequality,
    \begin{align}
        \Pb\left( \mathcal{F}_k \right) \leq \Pb\left( M_t\left( \lambda_k \right) \geq T^{\alpha+1} \right) \leq \frac{\Eb\left[ M_t(\lambda_k) \right]}{T^{\alpha+1}}. \label{eq:boundFk}
    \end{align}
    Next, it remains to show that $M_t (\lambda)$ is a martingale such that $\Eb[M_t (\lambda)] = 1$. Let $\mathcal{H}_{t-1} = \{p_s, \epsilon_s \}_{s=1}^{t-2}$ be the history of selected priors and noise up to and including time $t-2$, then 
    \begin{align}
        \E[M_t(\lambda) | \mathcal{H}_{t-1}] &= M_{t-1}(\lambda) \cdot \Eb \left[ \exp \left( \lambda \indfcn{p_{t-1}=p} \epsilon_{t-1} - \lambda^2\sigma^2 \indfcn{p_{t-1} = p} / 2 \right) \big| \mathcal{H}_{t-1} \right] \\
        &= M_{t-1}(\lambda) \cdot \Big( 
        \underbrace{\Eb \left[ 
            \exp(0) | p_{t-1} \neq p,\mathcal{H}_{t-1} \right]}_{=1} \Pb(p_{t-1} \neq p | \mathcal{H}_{t-1} )
        \\ &\quad+\underbrace{\Eb \left[
            \exp(\lambda \epsilon_{t-1} - \lambda^2 \sigma^2/2) 
            | p_{t-1} = p, \mathcal{H}_{t-1} 
        \right]}_{= 1 \text{ since } \epsilon_{t-1} \ind p_{t-1}, \mathcal{H}_{t-1}} \Pb(p_{t-1} = p, \mathcal{H}_{t-1}) \Big)
        \\
        &= M_{t-1}(\lambda).
    \end{align}
    Applying the above recursively to $\Eb[M_t (\lambda)]$ and defining $M_{1}(\lambda) = 1$, we get that $\Eb[M_t (\lambda)]  =1$. Thus, from \cref{eq:boundFk} and a union bound over $k \in [T-1]$, $\Pb(S_{t,p} > \sigma \sqrt{2 (\alpha+1) N_{t-1}(p) \log (T)}) \leq 1 / T^\alpha$. By symmetry of the Gaussian noise, $\Pb(|S_{t,p}| > \sigma \sqrt{2(\alpha+1) N_{t-1}(p) \log (T)}) \leq 2 / T^\alpha$.    
\end{proof}

Finally, we are ready to state and prove the regret bound for HP-GP-TS.

\thmHPGPTS*
\begin{proof}
Recall that $p^*, x^* | H_t \eqd p_t, x_t | H_t$ and that $U_{t,p}(x)$ is a deterministic function w.r.t. $p$ and $x$ conditioned on the history $H_t$, therefore $\E[U_{t,p^*}(x^*)] = \E[U_{t,p_t}(x_t)]$ follows by the tower rule. We begin by decomposing the Bayesian regret into three terms and show the bounds that we will later obtain for each of them. 
\begin{align}
    \BR(T) &= \sumt \E[ f(x^*) - f(x_t)]\\
    &= \sumt \E[ f(x^*) - U_{t,p^*}(x^*) + U_{t,p_t}(x_t) - f(x_t)] 
    \hspace{3.5em} \left( p^*, x^* | H_t \eqd p_t, x_t | H_t \right) \\
    &\begin{multlined}[b]
    = \sumt \E\bigg[ f(x^*) - U_{t,p^*}(x^*) + (\sqrt{\beta_t} + \sqrt{\eta_T})\sigma_{t,p_t}(x_t) \\ 
    + \mu_{t,p_t}(x_t) - \sqrt{\eta_T} \sigma_{t,p_t}(x_t) - f(x_t)\bigg] 
    \end{multlined}
    \hspace{3em} \left( \pm \sqrt{\eta_T} \sigma_{t,p_t}(x_t) \right)
    \\
    &= \underbrace{\sumt \E\left[ f(x^*) - U_{t,p^*}(x^*) \right]}_{A_1} 
    + \underbrace{\sumt \E \left[ (\sqrt{\beta_t} + \sqrt{\eta_T})\sigma_{t,p_t}(x_t) \right]}_{A_2} 
    \\ &\quad 
    + \underbrace{\sumt \E \left[ \mu_{t,p_t}(x_t) - \sqrt{\eta_T} \sigma_{t,p_t}(x_t) - f(x_t) \right ]}_{A_3}. \\
    &\begin{multlined}[b]
    \leq 
    \underbrace{\frac{\pi^2}{6}}_{A_1} 
    + \underbrace{\sqrt{CT \bar{\gamma}_T(P_1)} (\sqrt{\beta_T} + \sqrt{\eta_T})}_{A_2} 
    \\+ \underbrace{
    \frac{\sqrt{3}}{T}\left( \sqrt{\sigma^{-2}( \bar{M} + 2M \mumax + \mumax^2 + \sigma^2)} +\sqrt{\bar{M} / T}\right) 
    }_{A_{3,1}}
    \\ + \underbrace{
        \sigma \sqrt{14 T|P| \log T} + |P| \left( \sigma^{-1} \sqrt{T}(M + \mumax + \sigma)
        + M + \sigma \sqrt{2 \log T} \right).
    }_{A_{3,2}} 
    \end{multlined}
\end{align}
Next, we will prove the bounds for the terms $A_1$, $A_2$, and $A_3$ where the bound for $A_3$ is given by the sum of $A_{3,1}$ and $A_{3,2}$.

\paragraph{Bounding $A_1$}
Since the upper confidence term in $A_1$, $U_{t,p^*}(x^*)$, corresponds to the confidence bound of the true prior $p^*$, the bound for $A_1$ follows by standard techniques \citep{russoLearningOptimizePosterior2014}:
\begin{align}
    A_1 &= \sumt \E \left[ f(x^*) - U_{t,p^*}(x^*) \right] \\
    &\leq \sumt \E \left[ \left[ f(x^*) - U_{t,p^*}(x^*) \right]_{+} \right]
    \hspace{14.25em} \left( [\cdot]_{+} := \max(\cdot, 0 ) \right) \\
    &\leq \sumt \sum_{x \in \Xc} \E \left[ \left[ f(x) - U_{t,p^*}(x) \right]_{+} \right]
    \hspace{13.25em} \left(x^* \in \Xc, [\cdot]_+ \geq 0\right) \\
    &= \sumt \sum_{x \in \Xc} \E_{p^*, H_t} \left[ \E_{t} \left[ \left[ f(x) - \mu_{t, p^*}(x) - \sqrt{\beta_t} \sigma_{t, p^*}(x) \right]_+ \big|\, p^*, H_t \right] \right]. 
    \hspace{1em} (\text{Tower rule})
\end{align}
Recall that for $Z \sim \N(\mu, \sigma)$ with $\mu \leq 0$, $\E[ [Z]_+ ] \leq \frac{\sigma}{\sqrt{2\pi}} \exp\left( \frac{-\mu^2}{2 \sigma^2} \right)$. 
In our case, note that $f(x) | p^*, H_t \sim \N(\mu_{t, p^*}(x), \sigma^2_{t, p^*}(x))$ and $-\mu_{t,p^*}(x) - \sqrt{\beta_t} \sigma_{t, p^*}(x)$ is deterministic given $p^*, H_t$. Hence,
\begin{align}
    A_1 &\leq \sumt \sum_{x \in \Xc} \E_{p^*, H_t} \left[ \frac{\sigma_{t, p^*}(x)}{\sqrt{2\pi}} \exp\left( \frac{- \beta_t}{2} \right) \right] \\
    &\leq \sumt \sum_{x \in \Xc} \E_{p^*, H_t} \left[ \frac{1}{\sqrt{2\pi}} \exp \left( \frac{-\beta_t}{2} \right) \right] 
    &(\sigma_{t, p^*}(x) \leq \sigma_{0, p^*}(x) \leq 1) \\
    &= \sumt \sum_{x \in \Xc} \frac{1}{\sqrt{2\pi}} \exp( -\beta_t / 2) \\
    &= \sumt \frac{1}{t^2} \leq \frac{\pi^2}{6}.  
    &( \beta_t = 2 \log(|\Xc| t^2 / \sqrt{2\pi}) )
\end{align}

\paragraph{Bounding $A_2$}
To bound $A_2$, we separate it into two terms and apply Cauchy-Schwarz to each term:
\begin{align}
    A_2 &= \E \left[ \sumt \sqrt{\beta_t}\sigma_{t,p_t}(x_t) \right] + \E \left[ \sumt \sqrt{\eta_T}\sigma_{t,p_t}(x_t) \right] \\
    &\leq \E \left[ \sqrt{ \sumt \beta_t \sumt \sigma_{t,p_t}^2(x_t)} \right] + \E \left[ \sqrt{ \sumt \eta_T \sumt \sigma^2_{t,p_t}(x_t) }\right] 
    &(\text{Cauchy-Schwarz}) \\
    &\leq \E \left[ \sqrt{ T \beta_T \sumt \sigma_{t,p_t}^2(x_t)} \right] + \E \left[ \sqrt{ T \eta_T \sumt \sigma^2_{t,p_t}(x_t) }\right]
    &(\beta_t \leq \beta_T, \forall t \in [T]) \\
    &=  \sqrt{T} (\sqrt{\beta_T} + \sqrt{\eta_T}) \E \Bigg[ \sqrt{ \sumt \sigma_{t,p_t}^2(x_t)} \Bigg] \\
    &\leq  \sqrt{T} (\sqrt{\beta_T} + \sqrt{\eta_T}) \sqrt{\E \Big[  \sumt \sigma_{t,p_t}^2(x_t) \Big]}.
    &(\text{Jensen's inequality})
\end{align}

By \cref{lem:infgain_selected_leq_avg_infgain}, we have that $\E\big[\sumt \sigma_{t,p_t}^2(x_t) \big] \leq C \bar{\gamma}_T(P_1)$ and therefore
\begin{align}
    A_2 \leq \sqrt{C T \bar{\gamma}_T(P_1)} \left(\sqrt{\beta_T} + \sqrt{\eta_T}\right).
\end{align}

\paragraph{Bounding $A_3$}
We further split $A_3$ based on whether $p_t \in \C_t$ holds: 
\begin{align}
    A_3 &= \underbrace{\sumt \E \left[ (\mu_{t,p_t}(x_t) - \sqrt{\eta_T} \sigma_{t,p_t}(x_t) - f(x_t)) \indfcn{p_t \notin \C_t} \right ]}_{A_{3,1}} \\
    & \phantom{=} + \underbrace{\sumt \E \left[ (\mu_{t,p_t}(x_t) - \sqrt{\eta_T} \sigma_{t,p_t}(x_t) - f(x_t)) \indfcn{p_t \in \C_t} \right ]}_{A_{3,2}}.
\end{align}
\paragraph{Bounding $A_{3,1}$:} To bound $A_{3,1}$, we apply the Cauchy-Schwarz inequality for expectations to separate the factors $\mu_{t,p_t}(x_t) - f(x_t)$ and $\1\{p_t \notin \C_t\}$ into different expectations as follows:
\begin{align}
    A_{3,1} &\leq \sumt \E \left[ (\mu_{t,p_t}(x_t) - f(x_t)) \indfcn{p_t \notin \C_t} \right ] 
    \hspace{9em} \left( \sqrt{\eta_T} \sigma_{t,p_t}(x_t) \geq 0 \right) \\
    &= \sumt \left( \E \left[ \mu_{t,p_t}(x_t) \indfcn{p_t \notin \C_t} \right ] +
    \E \left[ - f(x_t) \indfcn{p_t \notin \C_t} \right ] \right) \\
    &\begin{multlined}[b]
        \leq \sumt \bigg( \sqrt{\E \left[ (\mu_{t,p_t}(x_t))^2 \right] \E \left[ (\indfcn{p_t \notin \C_t})^2 \right ] } \\ 
        + \sqrt{\E \left[ (f(x_t))^2 \right] \E \left[ (\indfcn{p_t \notin \C_t})^2 \right] } \bigg)
    \end{multlined}
    \hspace{3.5em} \left( \E[XY] \leq \sqrt{\E[X^2] \E[Y^2]} \right)\\
    &\leq \sumt \sqrt{ \E \left[ \indfcn{p^* \notin \C_t} \right ]} \left( \sqrt{\E \left[ (\mu_{t,p_t}(x_t))^2 \right]} + \sqrt{\E \left[ \big(\sup_{x \in \Xc} |f(x)| \big)^2 \right]} \right).
    \hspace{0.5em} (p^* | H_t \eqd p_t | H_t)
\end{align}
To bound the three expectations above, we have from \cref{lem:true_prior_in_conf_set} that $\E \left[ \indfcn{p^* \notin \C_t} \right ] \leq 3 T^{-5}$ and from \cref{lem:expected_mean_bound} that $\E \left[ (\mu_{t,p_t}(x_t))^2 \right] \leq \sigma^{-2}T( \bar{M} + 2M \mumax + \mumax^2 + \sigma^2)$ for all $t \in [T]$. Similarly, by \cref{lem:bound_sup_f_squared} we have that $\E \left[ \big(\sup_{x \in \Xc} |f(x)| \big)^2 \right] \leq \bar{M} $. Put together, we arrive at the following bound for $A_{3,1}$:
\begin{equation}
    A_{3,1} 
    \leq \frac{\sqrt{3}}{T} \left( \sqrt{\sigma^{-2}( \bar{M} + 2M \mumax + \mumax^2 + \sigma^2)} + \sqrt{ \bar{M} /T } \right).
\end{equation}

\paragraph{Bounding $A_{3,2}$:} Then, $A_{3,2}$ can be bound as follows:
\begin{align}
    A_{3,2} &= \sumt \E \left[ \left(\mu_{t,p_t}(x_t) - \sqrt{\eta_T} \sigma_{t,p_t}(x_t) - f(x_t) - \epsilon_t \right) \indfcn{p_t \in \C_t} \right ] 
    \hspace{1.5em} \left( \epsilon_t \ind \indfcn{p_t \in \C_t}\right) \\
    &\leq \sum_{p \in P} \E \left[ \sumt \left(\mu_{t,p}(x_t) - \sqrt{\eta_T} \sigma_{t,p}(x_t) - f(x_t) - \epsilon_t \right) \indfcn{p_t = p} \indfcn{p \in \C_t} \right] \label{eq:prior-confidence-set}
\end{align}
We define the final time step where prior $p$ is selected and is in the confidence set as $\tau_p := \max \left\{t \in [T] : p_t = p, p \in \C_t \right\}$. Then, $\sum_{t \in [\tau_p-1]} (\mu_{t,p_t}(x_t) - \sqrt{\eta_T}\sigma_{t,p_t}(x_t) - f(x_t)-\epsilon_t)\indfcn{p_t = p}\indfcn{ p \in \C_t} = G_{\tau_p}(p)$ since $\C_t$ is a shrinking sequence of sets. By definition of $p \in \C_{\tau_p}$ (\cref{eq:confidence_set}), $G_{\tau_p}(p)  \leq \sigma \sqrt{14 N_{\tau_p-1}(p) \log T}$ and

\begin{align}
    A_{3,2} &\leq \sum_{p \in P} \E \left[ \sigma \sqrt{ 14 N_{\tau_p-1}(p) \log T} + \left(\mu_{\tau_p,p}(x_t) - \sqrt{\eta_T} \sigma_{\tau_p,p}(x_{\tau_p}) - f(x_{\tau_p}) - \epsilon_{\tau_p} \right) \right] \label{eq:post-confidence-set}\\
    &\leq \sum_{p \in P} \E \left[ \sigma \sqrt{14 N_T(p) \log T} \right] + \sum_{p \in P} \E \left[ \left(\mu_{\tau_p,p}(x_{\tau_p}) - f(x_{\tau_p}) - \epsilon_{\tau_p} \right) \right] \label{eq:noise_bound_and_final_excess}
\end{align}
since $\sqrt{\eta_T} \sigma_{t,p}(x) \geq 0,$ and $N_t(p) \leq N_T(p)$, $\forall t,p,x \in [T] \times P \times \Xc$. To bound the left term in \cref{eq:noise_bound_and_final_excess}, we apply the Cauchy-Schwarz inequality such that $\sum_{p \in P} \sqrt{N_T(p)} \leq \sqrt{T|P|}$. For the right term in \cref{eq:noise_bound_and_final_excess}, we note that $\tau_p \in [T]$ and consider the maximum:
\begin{align}
    &\E \left[ \left(\mu_{\tau_p,p}(x_{\tau_p}) - f(x_{\tau_p}) - \epsilon_{\tau_p} \right) \right] \\
    &\leq
    \sum_{p \in P} \E \left[ \max_{t \in [T]} \left(\mu_{t,p}(x_t) - f(x_{t}) - \epsilon_{t} \right) \right] \\
    &\leq  
    \sum_{p \in P} \left( \E \left[ \max_{t \in [T]} \mu_{t,p}(x_t) \right]
    + \E \left[ \sup_{x \in \Xc} |f(x)| \right] 
    + \E \left[ \max_{t \in [T]} - \epsilon_{t} \right] \right) \\
    &\leq 
    |P| \left( \sigma^{-1} \sqrt{T}(M + \mumax + \sigma)
    + M
    + \sigma \sqrt{2 \log T} \right). 
    && \left( \text{\cref{lem:expected_mean_bound}} \right)
\end{align}
The bound $\E[ \max_{t \in [T]} -\epsilon_t] \leq \sigma \sqrt{2 \log T}$ follows by standard results for independent zero-mean Gaussians \citep[Section 2.5]{boucheronConcentration2013}. Combined, we get that
\begin{equation}
    A_{3,2} \leq \sigma \sqrt{14 T|P| \log T} + |P| \left( \sigma^{-1} \sqrt{T}(M + \mumax + \sigma)
    + M
    + \sigma \sqrt{2 \log T} \right).
\end{equation}
\end{proof}
\begin{remark} \label{rem:time-uniform-confset}
    Unlike \citet{hongThompson2022,luSurrogate2023}, our definition of the confidence set $\C_t$ includes a condition that the excess reward $G_s(p)$ is below the threshold $\xi_s(p)$ for all $s \leq t$, not just $s = t$. This guarantees that the sets are non-increasing in size, and therefore if $p \in \C_t$ then $p \in C_s$ for all $s < t$. Furthermore, $\sum_{s = 1}^{t-1} (\mu_{s,p_s}(x_s) - \sqrt{\eta_T} \sigma_{s,p_s}(x_s) - y_s) \indfcn{p=p_s}\indfcn{p \in \C_s} = G_t(p)$ if $p \in \C_{t-1}$ which is critical to go from \cref{eq:prior-confidence-set} to \cref{eq:post-confidence-set}. Without the time-uniform requirement, $p \notin C_s$ could hold for some $s < t$ s.t. $\mu_{s,p_s}(x_s) - \sqrt{\eta_T} \sigma_{s,p_s}(x_s) - y_s < 0$. Then, $\sum_{s = 1}^{t-1} (\mu_{s,p_s}(x_s) - \sqrt{\eta_T} \sigma_{s,p_s}(x_s) - y_s) \indfcn{p=p_s}\indfcn{p \in \C_s} > G_t(p)$ which prevents bounding $G_t(p)$ by $\xi_t(p)$.
\end{remark}

\subsection{Auxiliary lemmas}
In this section, we state and prove auxiliary lemmas that bound the expectations of $\mu_{t,p}(x_t)$ and $\sup_{x \in \Xc} |f(x)|^2$.
\begin{lemma}\label{lem:expected_mean_bound}
Let $\mumax = \sup_{p,x \in P \times \Xc} \mu_{1,p}(x)$, $M = \E[ \sup_{x \in \Xc} |f(x)|]$, $M_p = \E[ \sup_{x \in \Xc} |f(x)| | p^* = p]$, $M_\Delta = \max_{p \in P} M_p - \min_{p \in P} M_p$, and $\bar{M} = M^2 + 1 + \frac{M_\Delta^2}{4}$. If $k_p(x,x): \Xc \times \Xc \mapsto [-1,1]$, $\forall p \in P$, then 
\begin{align}
    \E[\mu_{t,p_t}(x_t)^2] \leq \frac{T}{\sigma^2} ( \bar{M} + 2M \mumax + \mumax^2 + \sigma^2), \\
    \E[\maxt \mu_{t,p}(x_t)] \leq \frac{\sqrt{T}}{\sigma}(M + \mumax + \sigma).
\end{align}

\end{lemma}
\begin{proof}
To begin, recall that $\mu_{t,p}(x) = \k_{t,p}(x)^T \left( \K_{t,p} + \sigma^2 I \right)^{-1} (\f_{1:t-1} + \beps_{1:t-1} - \bmu_{1:t-1,p})$ where $\f_{1:t} = [f(x_1), \ldots, f(x_{t-1})]^T$, $\beps_{1:t} = [\epsilon_1, \ldots, \epsilon_{t-1}]$, and $\bmu_{1:t,p} = [\mu_{1,p}(x_1), \ldots, \mu_{1,p}(x_{t-1})]$. Additionally, we note that 
\begin{align}
\norm[\Big]{ \k_{t,p}(x)^T \left( \K_{t,p} + \sigma^2 I \right)^{-\frac{1}{2}} }_2^2 = \k_{t,p}(x)^T \left( \K_{t,p} + \sigma^2 I \right)^{-1} \k_{t,p}(x) \leq k_p(x, x) \leq 1,
\end{align}
$\forall t,p,x \in [T] \times P \times \Xc$ by the definition of the posterior variance $\sigma^2_{t,p}(x)$ and since the posterior variance is non-negative $\sigma^2_{t,p}(x) \geq 0$. Similarly, note that $\norm[\Big]{\left( \K_{t,p} + \sigma^2 I \right)^{-\frac{1}{2}}}_2 \leq \sigma^{-1}$ since $\K_{t,p}$ is positive semi-definite for any $t$ and $p$. Therefore,
\begin{align}
    \mu_{t,p}(x) &\leq \norm[\Big]{ \k_{t,p}(x)^T \left( \K_{t,p} + \sigma^2 I \right)^{-\frac{1}{2}} }_2 \cdot
    \norm[\Big]{\left( \K_{t,p} + \sigma^2 I \right)^{-\frac{1}{2}}}_2
    \cdot 
    \norm{\f_{1:t} + \beps_{1:t} - \bmu_{1:t,p}}_2 \\
    &\leq \frac{1}{\sigma} \norm{\f_{1:t} + \beps_{1:t} - \bmu_{1:t,p}}_2.
\end{align}

Then, we bound $\E[ \maxt \mu_{t,p}(x_t)]$. 
\begin{align}
    \E[ \maxt \mu_{t,p}(x_t)] 
    &\leq 
    \frac{1}{\sigma} \left(
     \E \left[ \maxt 
     \norm{\f_{1:t}}_2  
    +  \maxt 
     \norm{\beps_{1:t}}_2  
    + \maxt 
     \norm{\bmu_{1:t,p}}_2  
     \right] \right)
     \\
    &\leq 
    \frac{1}{\sigma}
     \E \left[ \maxt 
     \sqrt{\sum_{s=1}^t f(x_s)^2} 
     + \maxt 
        \sqrt{\sum_{s=1}^t \epsilon_s^2}
     + \maxt 
        \sqrt{\sum_{s=1}^t \mu_{1,p}^2(x_s)}
     \right]  \\
    &\leq 
    \frac{1}{\sigma} \left(
     \E \left[  
     \sqrt{T \sup_{x \in \Xc} f(x)^2} 
     \right]
    + \E \left[ 
        \sqrt{\sum_{t=1}^T \epsilon_t^2}
     \right]
     + \sqrt{ T \sup_{x \in \Xc} \mu_{1,p}^2(x)}
     \right) \\
    &\leq 
    \frac{1}{\sigma} \left(
    \sqrt{T}
     \E \left[  
     \sup_{x \in \Xc} |f(x)| 
     \right]
    + \sqrt{\E \left[ 
        \sum_{t=1}^T \epsilon_t^2
     \right]} 
     + \sqrt{T} \mumax \right) 
     \hspace{0.25em} (\text{Jensen's ineq.})\\
    &\leq 
    \frac{\sqrt{T}}{\sigma} \left(M + \sigma + \mumax \right).
     \hspace{10.5em} \left( \textstyle \sumt \epsilon_t^2 \sim \sigma^2 \chi^2_T \right)
\end{align}
Similarly, we bound $\E[\mu_{t,p}(x_t)^2]$:
\begin{align}
    \E[\mu_{t,p_t}(x_t)^2] &\leq \frac{1}{\sigma^2} \E \left[ \norm{\f_{1:t} + \beps_{1:t} - \bmu_{1:t,p_t}}_2^2 \right] \\
    &\leq \frac{1}{\sigma^2} \E \left[ \sumt (f(x_t) + \epsilon_t - \mu_{1,p_t}(x_t) )^2 \right] \\
    &\begin{multlined}[b]
    = \frac{1}{\sigma^2} \E \Bigg[ \sumt f(x_t)^2 +\epsilon_t^2 +\mu_{1,p_t}^2(x_t)
    \\
    \hspace{7em} + 2 ( f(x_t) \epsilon_t  - f(x_t)\mu_{1,p_t}(x_t) - \epsilon_t \mu_{1,p_t}(x_t)) \Bigg] 
    \end{multlined}
    \\
    &= \frac{1}{\sigma^2} \E \left[ \sumt 
        f(x_t)^2 + \epsilon_t^2 + \mu_{1,p_t}^2(x_t)- 2f(x_t) \mu_{1,p_t}(x_t)
    \right] 
    \hspace{0.0em} \left( \parbox{8.4em}{$f(x_t), \mu_{1,p_t}(x_t) \ind \epsilon_t$, \\and $\E[\epsilon_t] = 0$} \right) \\
    &\leq \frac{1}{\sigma^2} \E \left[ \sumt \sup_{x \in \Xc} f(x)^2 +\epsilon_t^2 + \mu_{1,p_t}^2(x_t) + 2 \sup_{x \in \Xc} |f(x)| |\mu_{1,p_t}(x)| \right] \\
    &\leq \frac{1}{\sigma^2} T \left( \E \left[  \left(\sup_{x \in \Xc} |f(x)|\right)^2 \right] + \sigma^2 + \mumax^2 + 2 M \mumax \right)
    \hspace{1.5em} \left( \textstyle \sumt \epsilon_t^2 \sim \sigma^2 \chi^2_T \right)\\
    &\leq \frac{1}{\sigma^2} T \left( \bar{M} + \sigma^2 + \mumax^2 + 2M \mumax \right).
    \hspace{2em} \left( \text{\cref{lem:bound_sup_f_squared}} \right)
\end{align}

\end{proof}

\begin{lemma} \label{lem:bound_sup_f_squared}
    Let $M = \E[\sup_{x \in \Xc} |f(x)|]$, $M_p = \E[\sup_{x \in \Xc} |f(x)| \big| p^* = p]$, and $M_\Delta = \max_{p \in P} M_p - \min_{p \in P} M_p$. If $k_p(x,x): \Xc \times \Xc \mapsto [-1,1]$, $\forall p \in P$, then
    \begin{align}
        \E \left[ \Big( 
            \sup_{x \in \Xc} |f(x)|
        \Big)^2 \right] \leq M^2 + 1 + \frac{M_\Delta^2}{4} =: \bar{M}.
    \end{align}
\end{lemma}
\begin{proof}
First, by the variance formula $\V(X) = \E[X^2] - \E[X]^2$,
\begin{align}
    \E[ (\sup_{x \in \Xc} |f(x)| )^2] = M^2 + \V\left( \sup_{x \in \Xc} |f(x)| \right).
\end{align}
The variance $\V[ \sup_{x \in \Xc} |f(x)|]$ can be bounded by the law of total variance as follows:
\begin{align}
    \V\left[\sup_{x \in \Xc} |f(x)|\right] &= \E_{p^*} \left[ \V \left( \sup_{x \in \Xc} |f(x)| \big| p^* \right)\right] + \V_{p^*} \Bigg( \underbrace{\E \left[ \sup_{x \in \Xc} |f(x)| \big| p^* \right]}_{M_{p^*}:=}\Bigg) \\
    &\overset{(a)}{\leq} \E_{p^*} \left[ \sup_{x \in \Xc} \sigma^2_{1,p^*}(x) \right]
    + \V_{p^*}(M_{p^*}) \\
    &\leq 1
    + \frac{(\max_p M_p - \min_p M_p)^2}{4} .
    \hspace{6em} \left( \parbox{9em}{$\sigma^2_{1,p}(x) \leq 1$,\\ and Popoviciu's ineq.} \right)
\end{align}
Note that $\V \left( \sup_{x \in \Xc} |f(x)| \,\big| \,p^* \right) \leq \sup_{x \in \Xc}\sigma_{1,p^*}^2(x)$, used in $(a)$, follows from the Gaussian Poincaré inequality applied to $\sup_{x \in \Xc} |f(x)|$, see \citet[Theorem 3.20 and Exercise 3.24]{boucheronConcentration2013}.
\end{proof}

\begin{figure}
    \centering
    \includegraphics[width=0.25\linewidth]{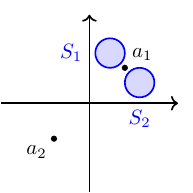}
    \caption{Potential counterexample to Eq (9) in \citet{hongThompson2022}. The blue regions represent when the event $E_0$ holds and the black dots represent the two arms $a_1$ and $a_2$.}
    \label{fig:hongcounterexample}
\end{figure}

\section{Technical issues with MixTS regret bound in the linear setting} \label{app:hongissues}
Theorem 1 in \citet{hongThompson2022} provides a regret bound for MixTS in the linear setting. The linear setting assumes that the true parameter $\theta^* | S_\ast \sim \N(\theta_{0, S_*}, \Sigma_{0, S_*})$ where the latent state $S_*$ is sampled from a discrete prior $P_1$. The proof of Theorem 1 in \citet{hongThompson2022} contains non-obvious steps that seem difficult to motivate. We use the notation of \citet{hongThompson2022}.

    First, Eq. (9) in \citet{hongThompson2022} uses the TS property that the true prior and optimal arm is equal in distribution to the selected prior and selected arm given the history: $A_{t,*}^\top \bar{\theta}_{t,S_*} | H_t \eqd A_{t}^\top \bar{\theta}_{t,S_t} | H_t$ (equivalent to $\mu_{t, p^*}(x^*) | H_t \eqd \mu_{t, p_t}(x_t) | H_t$ in our notation).
    However, Eq. (9) additionally conditions on the event $E_0 = \{ \| \theta_* - \theta_{0, S_*} \|_{\Sigma_{0,S_*}^{-1}} \leq \sqrt{2d \log (dn)} \}$ where $\theta_*$ lies close to its prior mean.
    The TS property does not hold under this event since it modifies the distribution of the linear parameter $\theta^*$ but not the sampled parameters $\theta_t$, thus changing the distribution of the optimal arm $A_{t,*}$ but not the selected arm $A_t$. Consider the example in \cref{fig:hongcounterexample}, if $E_0$ holds then $\theta_*$ lies in the blue regions and thus $a_2$ is optimal w.p. 0. If $E_0^c$ holds, then $a_2$ is optimal with a non-zero probability. However, MixTS is oblivious to $E_0$ given the history and thus $A_{t,*} | H_t, E_0 \not \eqd A_{t} | H_t$. This counterexample illustrates the overall idea but we have not validated that the scale of the arms and the blue regions are feasible.

    Second, five lines above Eq. (9) in \citet{hongThompson2022}, it is stated that the {\it regret} is upper-bounded by a constant $M$ whenever $E_0$ occurs. However, from Eq (9) to the first term in step 3 of their analysis (page 15), the bound of $M$ is applied implicitly to $A_t^\top \bar{\theta}_{t, S_t} - A_t^\top \theta_*$ without motivation. For the setting with bounded rewards, then $\bar{\theta}_{t, S_t}$ is also bounded but for Gaussian rewards $\bar{\theta}_{t, S_t}$ can be unbounded.

    Third, the second term in Eq. (9) contains the indicator function $\mathbf{1}\{E_0\}$: $\mathbb{E}[ (A_t^\top \bar{\theta}_{t, S_t} - A_t^\top \theta_*) \mathbf{1}\{ E_0 \} ]$. In step 3 (page 15), this indicator function is dropped without motivation: $\mathbb{E}[ \langle A_t^\top \bar{\theta}_{t, S_t} - A_t^\top \theta_* \rangle_M ]$ where $\langle \cdot \rangle_M = \min(\cdot, M)$ for the bound $M$. If the expression inside is non-negative w.p. 1, then this step would be valid but this is not the case. 

    Fourth, from our understanding, the final equation on page 15 adds and subtracts the confidence bound and adds a zero-mean Gaussian inside a minimum. However, adding a zero-mean Gaussian inside a minimum reduces the expectation but the analysis seems to assume that it would increase the expectation. I.e. it is seemingly assumed that $\mathbb{E}[\min(M, X)] \leq \mathbb{E}[\min(M, X + \epsilon_t)$ for a constant M and random variable $X$. However, the reverse inequality is true.

\section{Description of kernels} \label{app:kernels}
The RBF kernel, $k(x, \tilde{x}) = \exp(-|| x - \tilde x ||^2 / \ell^2)$ guarantees that $f$ is smooth. The lengthscale parameter $\ell > 0$ determines how quickly $f$ changes, smaller values lead to more fluctuations. The rational quadratic (RQ) kernel $k(x, \tilde x) = \left( 1 + \frac{|| x - \tilde x ||^2}{2\alpha \ell^2} \right)^{-\alpha}$ where $\alpha > 0$ is a mixture of RBF kernels with varying lengthscales.
The Matérn kernel \citep{maternSpatial1986} $k(x, \tilde x) = \frac{2^{1-\nu}}{\Gamma(\nu)} \left( \frac{\sqrt{2\nu}|| x - \tilde x ||}{\ell}\right)^{\nu} K_{\nu} \left( \frac{\sqrt{2\nu} ||x - \tilde x||}{\ell}\right)$ where $\nu > 0$ is the smoothness parameter that imposes that $f$ is k-times differentiable if $\nu > k$ for integer $k$. The functions $\Gamma(\nu)$ and $K_v$ correspond to the gamma function and a modified Bessel function \citep{williamsGaussianProcessesMachine2006}.
The periodic kernel $k(x, \tilde{x}) = \exp \left(-\frac{1}{2} \sum_{i=1}^d \sin^2 (\frac{\pi}{\rho}(x_i - \tilde x_i)) / \ell \right)$ generates smooth and periodic functions with period $\rho > 0$ \citep{mackayIntroductionGaussianProcesses1998}.
The linear kernel $k(x, \tilde{x}) = v x^\top \tilde x$ generates linear functions where $v$ is the variance parameter.

\section{Additional experimental details} \label{app:expdetails}
In this section, we provide some additional details about the experiments. All experiments were run on a compute cluster with a mix of GPUs (Nvidia A100, A40, T4 and V100). The GPU used was decided based on availability at the time and no implementation depends on a specific GPU. The algorithms were run in parallel in a single job for each seed. Each job in the synthetic and real-world data experiments ran for approximately 5 minutes. With the 500 seeds, this leads to a combined 250 GPU-hours. Running all algorithms for one seed in the lengthscale scaling experiment with $|P| = 128$ took approximately 40 minutes, and is in total equivalent to around 330 GPU-hours.

\subsection{Synthetic experiments}
For the kernel experiment, all kernels use a lengthscale of $1.0$ and are scaled s.t. $k(x, \tilde{x}) \leq 1$. In addition, the mean function for all priors is zero everywhere. 
For the subspace experiment, the total dimensions $d = 16$ but each prior $p_i$ assumes $f(x)$ depends on $d_s = 4$ subdimensions: $[i, i+1, i+2, i+3]$ for $i \in [5]$. Dimensions larger than $5$ are wrapped around 1, i.e. $((j-1) \mod 5) + 1$, such that the priors are equally difficult to distinguish and optimize. The prior elimination methods use $\delta = 0.05$ across all experiments, including the oracle methods. During every iteration $t$, SCoreBO samples $M$ priors from the hyperposterior $P_t$ and samples $N$ optimizers $x^*, f^*$ for each prior sampled through posterior sampling. 
In all experiments, we use $M = 16$ and $N = 12$ for SCoreBO. While our $M$ value matches that of \citet[Table 3]{hvarfner_self-correcting_2023}, we increase the $N$ value from 8 to 12. 
We use the implementation of the SCoreBO acquisition function in BoTorch (Community) \citep{balandat2020botorch}. To make the implementation fast with GPUs, we set \texttt{linear\_operator.settings.stable\_qr\_cpu\_threshold} to 8 in order to avoid QR-factorization being performed on CPU \citep{gardner_gpytorch_2018,pleiss2022linearoperator,linear_operator}. To avoid out of memory issues, we replace the default \texttt{torch.matmul} in \texttt{DefaultPredictionStrategy.\_exact\_predictive\_covar\_inv\_quad\_form\_root} (from \texttt{\hyphenchar\font=`\-gpytorch.models.exact\_prediction\_strategies}) with an equivalent \texttt{torch.einsum} \citep{gardner_gpytorch_2018}. Since the priors in our experiments are discrete, we compute the hyperposterior exactly and sample from it directly. Similarly, the expectation with respect to the hyperposterior is computed exactly for EEI. 

\subsection{Real-world data experiments}
As discussed in \cref{sec:experiments}, each dataset is split into a training and test set. The training sets are split into separate buckets to define our priors.
For each bucket $p$, we compute the empirical mean $\hat{\mu}_p$ and covariance $\hat{\Sigma}_p$ which defines the prior $\GP(\hat{\mu}_p, \hat{\Sigma}_p)$. The buckets in the Intel data corresponds to the 12 days in the training dataset. For the PeMS data, each hour between 06:00 and 13:00 defines one bucket, giving 7 priors. For the daily precipitation data, each month in the year constitutes a bucket, yielding 12 priors. When running the experiments, we select a measurement of all sensors from the test data uniformly at random. The selected measurements correspond to the unknown function $f(x)$ where $x$ is the sensor index and the goal is then to identify sensors measuring large temperatures, small speeds or high precipitation respectively for the three datasets. When the algorithms select an arm to evaluate, we add Gaussian noise with variance $\sigma^2$ around 5\% of the signal variance, similar to \citet{srinivasInformationTheoretic2012,bogunovic_time-varying_2016}.

For all the real-world datasets, sensors containing any null measurements are filtered out.

The Intel Berkeley dataset consists of measurements from 46 temperature sensors across 19 days. The training set consists of the first 12 days of measurements and the remaining 7 days constitute the test set. The noise variance is set to $\sigma^2 = 0.7^2$.

The PeMS data is considered in the public domain \citep{pems_tos} and consists of measurements from 211 sensors along the I-880 highway from all of 2023. The goal is to find the sensors with low speeds to identify congestions. We negate the speed values to obtain a maximization problem. We use the 5-min averages provided by PeMS. Data between 2023-01-01 and 2023-09-01 is put into the training set whilst the data until 2023-12-31 is put into the test set. The noise variance is set to $\sigma^2 = 2.25^2$.

The PNW precipitation data consists of daily precipitation data from 1949 to 1994 across 167 $50 \times 50$ km regions in the Pacific Northwest. The goal is to find the region with the highest precipitation for any given day. The training data consists of the measurements made prior to 1980 and the test data consists of the measurements between 1980 and 1994. The original data is stated to be given in mm/day however the data seems to be off by a factor of 10. We rescale the data to a log-scale using $\log (\cdot / 10 + 0.1)$, similar to \citet{krause_near-optimal_2008}. The noise variance is set to $\sigma^2 = 0.41^2$.

In the Intel experiment, we removed one outlier seed. All methods had a final cumulative regret around 6000°C on this instance, note that the average for the worst performing model across the other seeds was $\approx 250$°C. The outlier is shown in \cref{fig:outlier}. We can see that one of the sensors display very high temperatures compared to all other sensors, which is why all methods performed poorly on this seed. It should be noted that many of the sensors in the Intel data logged degrees above 100°C after a certain time - likely due to sensor failure rather than boiling temperatures in an office environment. Also note that these days were excluded from both our training and test data. The outlier could be an indication that this particular sensor was starting to fail earlier than others.

\begin{figure}[htb]
    \centering
    \includegraphics[width=0.5\linewidth]{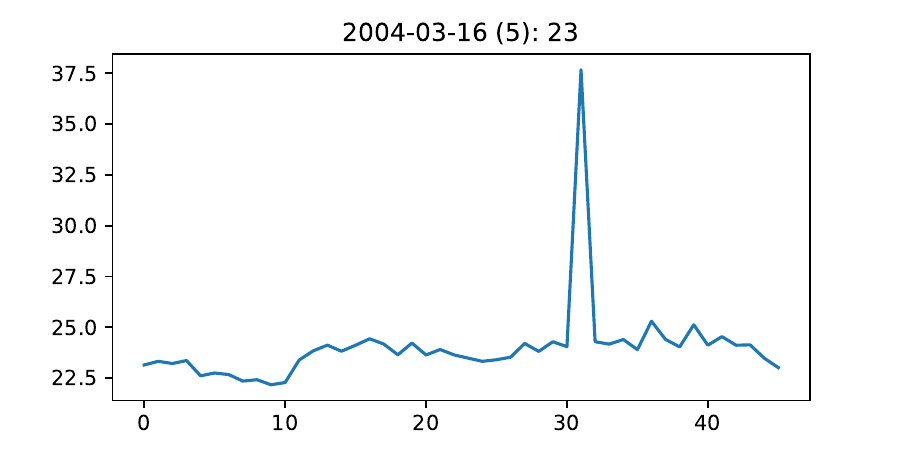}
    \caption{Removed sample from the test data in the Intel experiment. One of the sensors displays very high temperatures. }
    \label{fig:outlier}
\end{figure}

\section{Additional experimental results} \label{app:expresults}
In this section, we provide some additional experimental results.

\begin{table}
    \centering
    \caption{Average total regret and $\pm$ 1 standard error for the synthetic and real-world data experiments. The algorithms with the lowest regret (excluding oracle algorithms) are highlighted in bold.}
    \label{tab:jointregret}
    \resizebox{\linewidth}{!}{
    \begin{tabular}{c ccc ccc}
    \toprule
    \multirow{2}{*}{Algorithm} & \multicolumn{3}{c}{Synthetic} & \multicolumn{3}{c}{Real-world data} \\
     \cmidrule(r){2-4} \cmidrule(l){5-7}
     & Kernel & Lengthscale & Subspace & Intel & PeMS & PNW Precip. \\
    \midrule
MAP GP-TS & 84.3 $\pm$ 8.4 & \bf{30.2 $\pm$ 1.2} & \bf{87.2 $\pm$ 1.0} & 73.8 $\pm$ 7.7 & 1635.0 $\pm$ 129.3 & \bf{178.4 $\pm$ 6.9} \\ 
HP-GP-TS & \bf{39.2 $\pm$ 1.4} & \bf{31.4 $\pm$ 1.0} & \bf{88.3 $\pm$ 0.9} & \bf{54.1 $\pm$ 3.0} & \bf{1327.8 $\pm$ 107.9} & \bf{167.7 $\pm$ 5.2} \\ 
PE-GP-TS & 62.0 $\pm$ 0.6 & 61.8 $\pm$ 0.5 & 177.1 $\pm$ 1.4 & 106.5 $\pm$ 2.1 & \bf{1214.2 $\pm$ 81.5} & 200.9 $\pm$ 4.0 \\ 
PE-GP-UCB & 121.6 $\pm$ 1.2 & 114.2 $\pm$ 0.6 & 389.0 $\pm$ 1.5 & 173.0 $\pm$ 2.7 & 2159.2 $\pm$ 48.4 & 506.2 $\pm$ 2.6 \\ 
Oracle GP-TS & 35.0 $\pm$ 1.1 & 28.1 $\pm$ 0.8 & 86.0 $\pm$ 1.0 \\ 
Oracle GP-UCB & 68.5 $\pm$ 1.9 & 48.3 $\pm$ 1.2 & 217.3 $\pm$ 1.0 \\ 
SCoreBO & 180.4 $\pm$ 7.7 & 180.8 $\pm$ 5.8 & 106.6 $\pm$ 0.9 & 256.8 $\pm$ 9.3 & 3460.2 $\pm$ 163.7 & 861.3 $\pm$ 21.0 \\ 
EEI & \bf{39.0 $\pm$ 2.6} & \bf{30.1 $\pm$ 2.1} & \bf{88.3 $\pm$ 4.2} & \bf{51.6 $\pm$ 4.8} & 1664.1 $\pm$ 137.7 & 196.5 $\pm$ 12.3 \\ 
\bottomrule
    \end{tabular}
    }
\end{table}
First, we provide the average total regret for the synthetic and real-world data experiments in \cref{tab:jointregret}. We observe that HP-GP-TS either has the lowest regret or is within 1 standard error of the algorithm with the lowest regret across all the experiments. In \cref{fig:regretsynthetic}, it can be noted that EEI had low regret early in the synthetic experiment but HP-GP-TS either catches up or almost catches up later in the experiments. We compare both algorithms with an extended time horizon of $T = 1500$, the results are shown in \cref{fig:longregret,tab:longregret}. With the extended time horizon, HP-GP-TS achieves the lowest regret across all synthetic experiments. Although, EEI is still within 1 standard error on the lengthscale experiment.

\begin{figure}
    \centering
    \includegraphics[width=\linewidth]{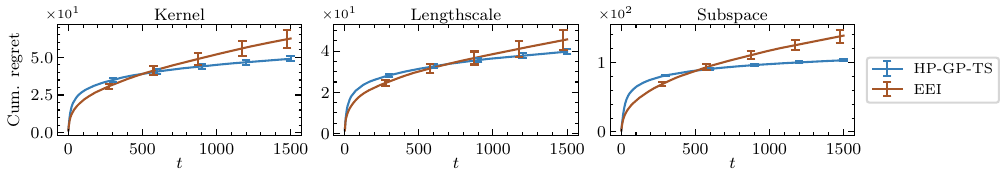}
    \caption{Cumulative regret for synthetic experiments extended time horizon $T = 1500$ with varying kernel (left), lengthscale (center) and mean function (right). Errorbars correspond to $\pm1$ standard error.}
    \label{fig:longregret}
\end{figure}

\begin{table}
    \centering
    \caption{Average total regret and $\pm$ 1 standard error for the synthetic experiments with longer horizon $T = 1500$. The algorithms with the lowest regret (excluding oracle algorithms) are highlighted in bold.}
    \label{tab:longregret}
    \begin{tabular}{c ccc}
        \toprule
        Algorithm & Kernel & Lengthscale & Subspace \\
        \midrule
HP-GP-TS & \bf{49.1 $\pm$ 1.6} & \bf{39.7 $\pm$ 1.2} & \bf{103.4 $\pm$ 1.3} \\ 
EEI & 62.6 $\pm$ 6.2 & \bf{45.7 $\pm$ 5.0} & 138.9 $\pm$ 9.2 \\ 

\bottomrule
    \end{tabular}
\end{table}

\begin{figure}
    \centering
    \includegraphics[]{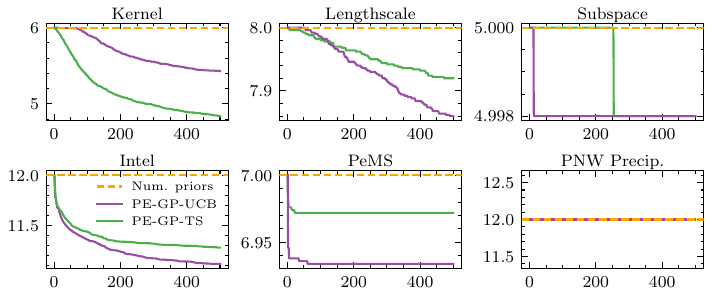}
    \caption{Mean number of priors remaining in $P_t$ over time for PE-GP-UCB and -TS.}
    \label{fig:elimination}
\end{figure}
\begin{figure}
    \centering    
    \includegraphics[width=\linewidth]{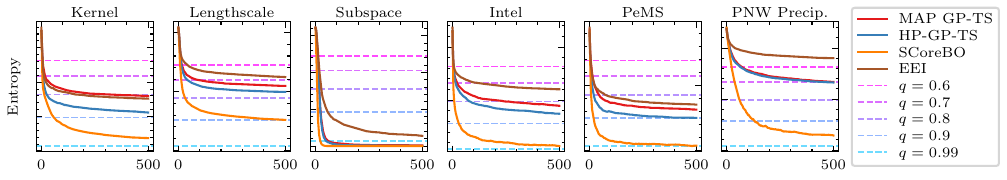}
    \caption{Average entropy in the hyperposterior $P_t$ over time for HP- and MAP GP-TS. The dashed reference values correspond to entropies of discrete distributions with prob. $q$ on one choice and prob. $\frac{1-q}{|P|-1}$ on the other $|P|-1$ choices.}
    \label{fig:entropy}
\end{figure}
Next, we include the mean number of priors in $P_t$ for all experiments in \cref{fig:elimination}. Similarly, we include the average entropy of the hyperposterior for all experiments in \cref{fig:entropy}. For the lengthscale, subspace, PeMS and PNW precipitation experiments, hardly any priors are eliminated. In contrast, the hyperposterior entropy concentrates rapidly across all experiments with the subspace and PNW precipitation having the most and least concentrated hyperposteriors.

We include the full set of confusion matrices for the lengthscale and subspace experiments in \cref{fig:confusionmatrices}. In the lengthscale experiments, we observe that PE-GP-UCB and -TS oversample the shortest lengthscale. This is similar to the kernel experiment where the Matérn 3/2 kernel was also oversampled. However, we see that HP-GP-TS and MAP GP-TS do not suffer from this optimistic bias. In the subspace experiment, HP- and MAP GP-TS have an accuracy of around 96\% whereas PE-GP-TS and -UCB have accuracies 30\% and 36\% respectively. Even though PE-GP-UCB has a higher accuracy than PE-GP-TS, it still has significantly higher regret. Additionally, the priors are equivalent up to coordinate permutations and therefore generate functions that are equally difficult to optimize. Unlike the kernel and lengthscale experiments, the PE-methods do not oversample any specific prior but commit too much time to exploring along the irrelevant dimensions.

In \cref{tab:lengthscales,tab:subspaces}, the total regret for the lengthscale and subspace scaling experiments are shown.

\begin{figure}
    \centering
    \begin{subfigure}{\linewidth}
    \centering
    \includegraphics[width = \linewidth]{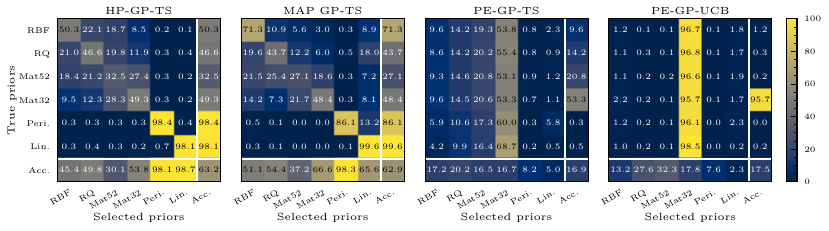}%
    \caption{Kernel experiment}
    \end{subfigure}
    \begin{subfigure}{\linewidth}
    \centering
    \includegraphics[width = \linewidth]{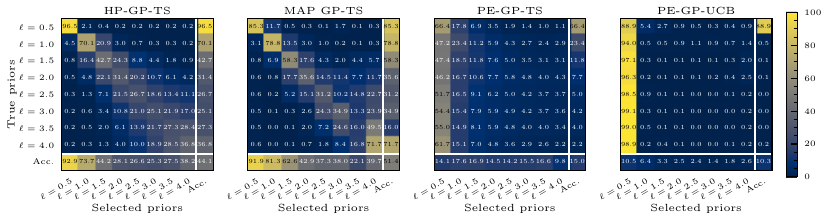}%
    \caption{Lengthscale experiment}
    \end{subfigure}
    
    \begin{subfigure}{\linewidth}
    \centering
    \includegraphics[width = \linewidth]{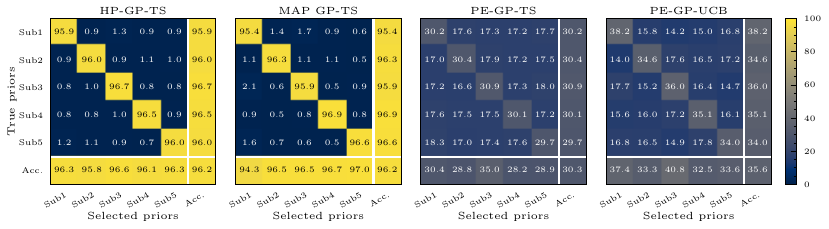}%
    \caption{Subspace experiment}
    \end{subfigure}
    \caption{Confusion matrices for the true prior $p^*$ and $p_t$ across all time steps of the synthetic experiments.}
    \label{fig:confusionmatrices}
\end{figure}

\begin{table}
    \centering
    \caption{Average total regret and $\pm1$ standard error for the lengthscale experiment as $|P|$ increases. The algorithms with the lowest regret (excluding oracle algorithms) are highlighted in bold.}
    \label{tab:lengthscales}
    \begin{tabular}{c ccccc}
        \toprule
        \multirow{2}{*}{Algorithm}& \multicolumn{5}{c}{Lengthscales, $|P|$} \\ 
        \cmidrule{2-6} & 8 & 16 & 32 & 64 & 128 \\ \midrule
MAP GP-TS & \bf{30.2 $\pm$ 1.2} & \bf{32.4 $\pm$ 2.5} & \bf{32.5 $\pm$ 2.1} & \bf{28.7 $\pm$ 1.1} & \bf{30.8 $\pm$ 1.9} \\ 
HP-GP-TS & \bf{31.4 $\pm$ 1.0} & \bf{31.7 $\pm$ 0.9} & \bf{30.8 $\pm$ 0.8} & \bf{30.7 $\pm$ 1.0} & \bf{31.0 $\pm$ 1.4} \\ 
PE-GP-TS & 61.8 $\pm$ 0.5 & 61.3 $\pm$ 0.5 & 62.2 $\pm$ 0.5 & 62.4 $\pm$ 0.4 & 64.3 $\pm$ 0.4 \\ 
PE-GP-UCB & 114.2 $\pm$ 0.6 & 114.8 $\pm$ 0.6 & 115.5 $\pm$ 0.6 & 114.5 $\pm$ 0.6 & 114.8 $\pm$ 0.6 \\ 
Oracle GP-TS & 28.1 $\pm$ 0.8 & 26.4 $\pm$ 0.8 & 27.3 $\pm$ 0.8 & 26.5 $\pm$ 0.7 & 25.7 $\pm$ 0.7 \\ 
Oracle GP-UCB & 48.3 $\pm$ 1.2 & 46.9 $\pm$ 1.1 & 48.4 $\pm$ 1.1 & 46.5 $\pm$ 1.0 & 45.6 $\pm$ 1.0 \\ 
SCoreBO & 180.8 $\pm$ 5.8 & 240.3 $\pm$ 6.3 & 277.2 $\pm$ 6.8 & 281.6 $\pm$ 6.7 & 283.9 $\pm$ 6.8 \\ 
EEI & \bf{30.1 $\pm$ 2.1} & \bf{30.9 $\pm$ 2.2} & \bf{29.4 $\pm$ 2.2} & \bf{30.8 $\pm$ 2.6} & \bf{32.1 $\pm$ 2.8} \\ 

         \bottomrule
    \end{tabular}
 \end{table}
\begin{table}
    \centering
    \caption{Average total regret and $\pm1$ standard error for the subspace experiment as $|P|$ increases. The algorithms with the lowest regret (excluding oracle algorithms) are highlighted in bold.}
    \label{tab:subspaces}
    
    \begin{tabular}{c cccc}
        \toprule
        \multirow{2}{*}{Algorithm}& \multicolumn{4}{c}{Subspaces, $|P|$} \\ 
        \cmidrule{2-5} & 5 & 8 & 12 & 16 \\ \midrule
MAP GP-TS & \bf{87.2 $\pm$ 1.0} & 89.9 $\pm$ 1.1 & 89.1 $\pm$ 0.9 & 90.9 $\pm$ 1.2 \\ 
HP-GP-TS & \bf{88.3 $\pm$ 0.9} & 88.8 $\pm$ 0.9 & 89.5 $\pm$ 0.9 & 90.8 $\pm$ 0.9 \\ 
PE-GP-TS & 177.1 $\pm$ 1.4 & 269.5 $\pm$ 1.9 & 344.7 $\pm$ 2.3 & 396.9 $\pm$ 2.5 \\ 
PE-GP-UCB & 389.0 $\pm$ 1.5 & 526.0 $\pm$ 1.8 & 622.4 $\pm$ 2.3 & 688.0 $\pm$ 2.7 \\ 
Oracle GP-TS & 86.0 $\pm$ 1.0 & 84.1 $\pm$ 0.9 & 84.6 $\pm$ 1.0 & 84.8 $\pm$ 1.0 \\ 
Oracle GP-UCB & 217.3 $\pm$ 1.0 & 218.2 $\pm$ 1.0 & 218.6 $\pm$ 1.0 & 218.9 $\pm$ 0.9 \\ 
SCoreBO & 106.6 $\pm$ 0.9 & 108.2 $\pm$ 0.8 & 108.9 $\pm$ 0.7 & 109.5 $\pm$ 0.7 \\ 
EEI & \bf{88.3 $\pm$ 4.2} & \bf{81.3 $\pm$ 3.8} & \bf{82.5 $\pm$ 3.6} & \bf{81.3 $\pm$ 3.8} \\ 

         \bottomrule
    \end{tabular}
\end{table}

\end{document}